%% file: main.tex
\documentclass[10pt,twocolumn,letterpaper]{article}

\usepackage{iccv}              

\input{preamble}

\definecolor{iccvblue}{rgb}{0.21,0.49,0.74}
\usepackage[pagebackref,breaklinks,colorlinks,allcolors=iccvblue]{hyperref}

\def\paperID{9201} 
\def\confName{ICCV}
\def\confYear{2025}

\title{Towards Higher Effective Rank in Parameter-efficient Fine-tuning \\using Khatri--Rao Product}

\author{Paul Albert\textsuperscript{1}\thanks{Work done at the Australian Institute for Machine Learning.} \quad Frederic Z. Zhang\textsuperscript{1}\footnotemark[1] \quad Hemanth Saratchandran\textsuperscript{2}\\Anton van den Hengel\textsuperscript{2} \quad Ehsan Abbasnejad\textsuperscript{3} \\
{\fontsize{11}{15}\selectfont \textsuperscript{1}Amazon Machine Learning, Australia\quad \textsuperscript{2}Australian Institute for Machine Learning\quad\textsuperscript{3}Monash University} \\
\normalsize \texttt{\{albrtpa, fredzz\}@amazon.com} \quad
{\tt\small \href{https://github.com/PaulAlbert31/KRAdapter}{https://github.com/PaulAlbert31/KRAdapter}}
}

\begin{document}
\maketitle
\input{sec/0_abstract}
\input{sec/1_intro}
\input{sec/2_formatting}
{
    \small
    \bibliographystyle{ieeenat_fullname}
    \bibliography{main}
}

\input{sec/3_appendix}
\end{document}

%% file: preamble.tex
\usepackage{microtype}
\usepackage{graphicx}
\usepackage{booktabs} 
\usepackage{pgfplots}
\usepackage{pgfplotstable}
\pgfplotsset{compat=1.7}
\usepackage{multirow}
\usepackage{tabularx}

\definecolor{color1}{HTML}{1f77b4}
\definecolor{color2}{HTML}{ff7f0e}
\definecolor{color3}{HTML}{2ca02c}
\definecolor{color4}{HTML}{e4bcad}
\definecolor{color5}{HTML}{c80064}
\definecolor{color6}{HTML}{76c8c8}
\definecolor{color7}{HTML}{3c4e4b}
\definecolor{color8}{HTML}{7f7f7f}
\definecolor{color9}{HTML}{bcbd22}
\definecolor{color10}{HTML}{17becf}
\definecolor{color11}{HTML}{aec7e8}
\definecolor{color12}{HTML}{ffbb78}
\definecolor{color13}{HTML}{98df8a}
\definecolor{color14}{HTML}{ff9896}
\definecolor{color15}{HTML}{c5b0d5}
\definecolor{color16}{HTML}{c49c94}
\definecolor{color17}{HTML}{f7b6d2}
\definecolor{color18}{HTML}{c7c7c7}
\definecolor{color19}{HTML}{dbdb8d}
\definecolor{color20}{HTML}{9edae5}
\definecolor{color21}{HTML}{ad494a}

\usepackage{amsmath}
\usepackage{amssymb}
\usepackage{mathtools}
\usepackage{amsthm}
\usepackage{amsfonts,mathrsfs}

\numberwithin{equation}{section}
\theoremstyle{plain}
\newtheorem{theorem}{Theorem}[section]

\newtheorem{lemma}[theorem]{Lemma}

\theoremstyle{definition}
\newtheorem{definition}[theorem]{Definition}

\theoremstyle{remark}

\usepackage{rotating}

\newcommand\R{\mathbb{R}}

%% file: sec/0_abstract.tex
\begin{abstract}
Parameter-efficient fine-tuning (PEFT) has become a standard approach for adapting large pre-trained models. Amongst PEFT methods, low-rank adaptation (LoRA) has achieved notable success. However, recent studies have highlighted its limitations compared against full-rank alternatives, particularly when applied to multimodal and large language models.
In this work, we present a quantitative comparison amongst full-rank and low-rank PEFT methods using a synthetic matrix approximation benchmark with controlled spectral properties. Our results confirm that LoRA struggles to approximate matrices with relatively flat spectrums or high frequency components---signs of high effective ranks.
To this end, we introduce KRAdapter, a novel PEFT algorithm that leverages the Khatri-Rao product to produce weight updates, which, by construction, tends to produce matrix product with a high effective rank.
We demonstrate performance gains with KRAdapter on vision-language models up to 1B parameters and on large language models up to 8B parameters, particularly on unseen common-sense reasoning tasks.
In addition, KRAdapter maintains the memory and compute efficiency of LoRA,
making it a practical and robust alternative to fine-tune billion-scale parameter models.
\end{abstract}

%% file: sec/1_intro.tex
\begin{figure*}[h]
\includegraphics[width=\textwidth]{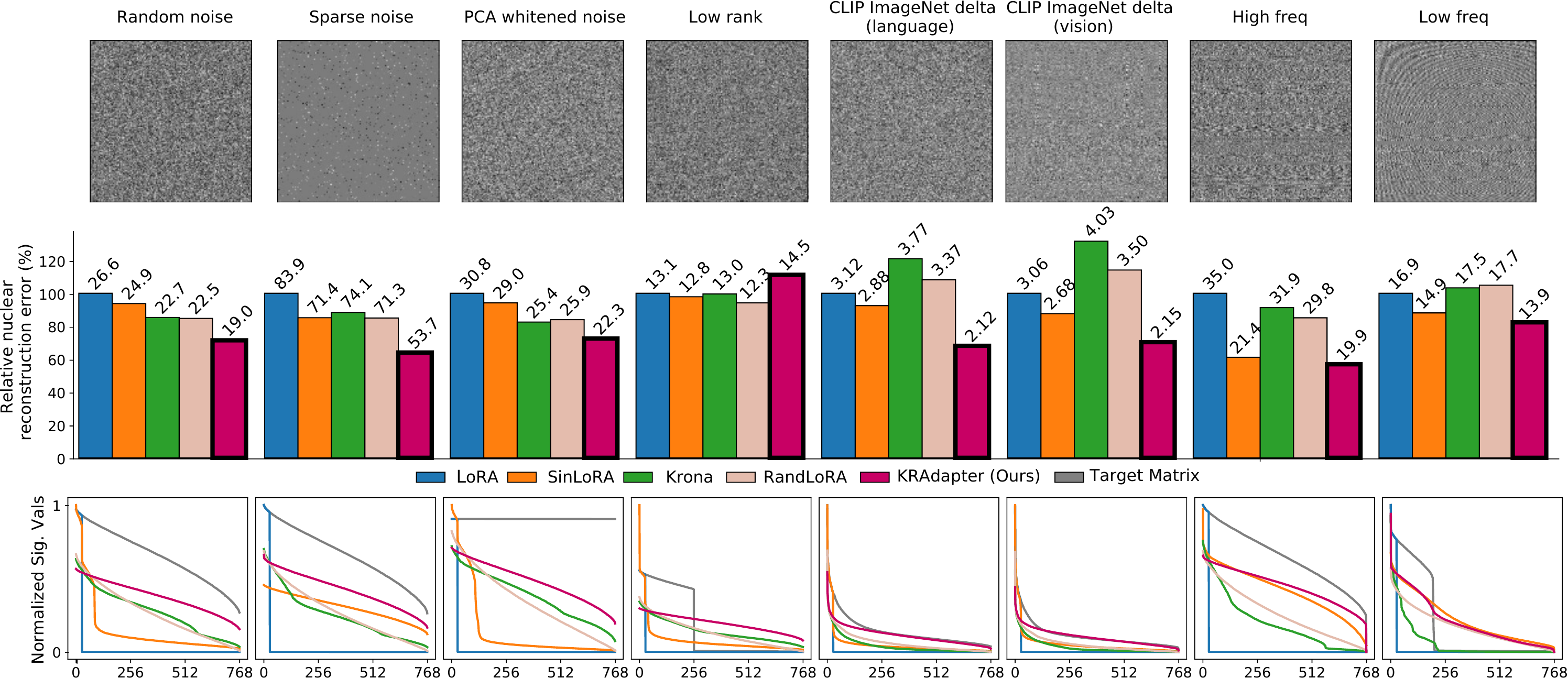}    
\caption{A visualization of the capacity of various PEFT methods to approximate the spectrum of different types of weight matrices. The first row shows the structure of the matrices.  In the second row we plot the squared nuclear error as a percentage of that of the LoRA solution, and add the absolute value above each bar. Lower is better. The third row shows the singular value distribution of the solutions compared to the target (square root scale). A flatter distribution indicates a higher effective rank. We compare LoRA~\cite{2022_ICLR_lora}, SinLoRA~\cite{2024_arxiv_sinelora}, Krona~\cite{2023_NeurIPSW_krona}, RandLoRA~\cite{2025_ICLR_RandLoRA} and our proposed algorithm KRAdapter. KRAdapter achieves a better approximation to the true matrix in every case except when we force it to have low rank (and thus to match the LoRA model).
\label{fig:results-toy-relative}}
\end{figure*}

\section{Introduction}
Large pre-trained models have lead to great success in both computer vision and natural language processing~\cite{2024_llama,2021_ICLR_vit,2021_ICML_CLIP,2023_arxiv_GPT4,2019_arxiv_roberta}.
However, fine-tuning these models for specific downstream tasks often demands substantial computational resources due to the sheer number of parameters.
Recently, parameter-efficient fine-tuning (PEFT) has emerged as an effective strategy to address this challenge,
enabling efficient adaptation using a fraction of the trainable parameters~\cite{2024_arxiv_peftsurvey}.
In particular, low-rank adaptation (LoRA)~\cite{2022_ICLR_lora}, a prominent PEFT technique, formulates the learnable weight updates as the product of two low-rank matrices, achieving remarkable memory efficiency with little performance drop.
Despite its success, recent studies~\cite{2025_ICLR_RandLoRA,2024_arxiv_sinelora} suggest that the low-rank formulation 
can limit its ability to effectively fine-tune models for complex tasks.
To shed light on this limitation, we compare LoRA to several full-rank PEFT methods by approximating weight matrices with controlled spectral properties.
As shown in Figure~\ref{fig:results-toy-relative}, full-rank methods can outperform LoRA in approximating both synthetic and real weight matrices.
Not all full-rank PEFT methods are equally effective---some produce approximations with rapidly decaying singular values, leading to a low \textit{effective rank}~\cite{2007_ESPC_effectiverank}.
Building on prior work~\cite{2019_ICML_batchnuclearnorm,2024_AAAI_svdregtesttime},
which links substantial tail singular values to improved generalization and robustness,
we hypothesize that a high effective rank is crucial to the performance of full-rank PEFT methods.

To this end, we introduce KRAdapter, a novel PEFT approach that leverages the Khatri-Rao matrix product.
By design, KRAdapter constructs provably full-rank weight matrices and empirically maintains large tail singular values.
The increased effective rank indicates that the weight updates span the parameter space more uniformly, allowing the model to capture complex patterns and generalize better to distribution shifts.
Across diverse architectures, including vision-language models and large language models, KRAdapter achieves performance gains while maintaining parameter efficiency. Furthermore, KRAdapter exhibits superior out-of-distribution (OOD) performance compared to LoRA and other full-rank methods. Our key contributions are: (1) a controlled matrix approximation benchmark comparing the representational advantages of full- and rank-constrained PEFT algorithms; (2) the identification of effective rank as a critical factor differentiating full-rank PEFT techniques; (3) KRAdapter, a novel Khatri-Rao product-based PEFT method that consistently achieves high effective rank; (4) empirical validation of its superior performance and out-of-distribution robustness compared to LoRA and other PEFT methods; and (5) evidence of KRAdapter's improved parameter scaling efficiency among full-rank approaches.

%% file: sec/2_formatting.tex
\section{Related Work}

The development of large pre-trained models has revolutionized various fields, yet their deployment and fine-tuning pose significant computational challenges \cite{2024_llama,2021_ICLR_vit,2021_ICML_CLIP,2023_arxiv_GPT4,2019_arxiv_roberta}. Parameter-Efficient Fine-Tuning (PEFT) methods have emerged as a crucial paradigm to address these challenges, allowing adaptation of these massive models to downstream tasks while training only a small fraction of the parameters \cite{2024_arxiv_peftsurvey}.

\subsection{Low-Rank Adaptation (LoRA)} LoRA \cite{2022_ICLR_lora} has become a cornerstone of PEFT due to its simplicity and effectiveness. LoRA introduces low-rank matrices to approximate the weight updates, significantly reducing the number of trainable parameters. Its success has led to numerous extensions and applications across various modalities \cite{2024_CVPR_promptvslora,2023_IJCAIW_LoRAAudio,2024_TMLR_multiloragen,2024_NeurIPS_atlas}. However, recent studies have begun to explore the limitations of the low-rank constraint, suggesting that it might hinder the model's ability to capture complex relationships, especially for demanding tasks \cite{2025_ICLR_RandLoRA, 2024_arxiv_sinelora,2023_NeurIPSW_krona}. Furthermore, the spectral properties of weight matrices, which influence a network's capacity for learning and generalization \cite{2017_arxiv_spectralnormreg,2018_NeurIPS_Spectralnormgan,2017_ICML_expressivepower}, may be constrained by the low-rank nature of LoRA updates. Our work builds upon these observations, confirming the need for methods that can capture richer feature interactions while maintaining parameter efficiency, potentially through weight updates with more desirable spectral characteristics.

\subsection{Enhancing the training efficiency}
To improve convergence speed and overall performance of LoRA, several improvements have been proposed. LoRA+~\cite{2024_ICML_lora+} enhances the training process by applying different learning rates to the low-rank matrices, allowing for more nuanced optimization. DoRA~\cite{2024_ICML_DoRA} decomposes the low-rank updates into magnitude and direction components, providing finer-grained control over the adaptation process. HydraLoRA~\cite{2024_NeurIPS_HydraLoRA} employs multiple upscaling matrices, acting as specialized "experts" focusing on distinct aspects of the input, potentially capturing more diverse feature interactions.

\subsection{Further reducing parameters}
While LoRA is inherently parameter-efficient, subsequent research has focused on pushing these limits further. Strategies to minimize the trainable parameter count beyond the rank-one case include methods leveraging linear combinations of fixed random matrices, such as VeRA~\cite{2024_ICLR_VeRA} and NoLA~\cite{2024_ICLR_NoLA}, where learned one-dimensional vectors combine pre-defined non-trainable random matrices. Another line of work explores initializing or fixing LoRA's projection matrices using insights from the pre-trained weights. Methods like SVFT~\cite{2024_ICMLW_SVFT} and Pissa~\cite{2024_NeurIPS_Pissa} fix or respectively initialize the projection matrices to the first eigenvectors corresponding to the largest singular values of the original weight matrix. Conversely, MiLoRA~\cite{2024_arxiv_milora} utilizes the last eigenvectors.

\subsection{Full-rank updates\label{sec:relatedfullrank}}
To address the representational bottlenecks and potential spectral limitations inherent in low-rank approximations, several methods explore parameter-efficient ways to perform full-rank updates. The goal is to capture more complex relationships present in the data without incurring the full cost of fine-tuning all parameters. Kronecker Adapters (KronA) \cite{2023_NeurIPSW_krona} utilize the Kronecker product to construct full-rank update matrices from smaller trainable factors. SinLoRA \cite{2024_arxiv_sinelora} applies a sine function parameterized by a frequency parameter on top of a low-rank update, effectively producing a full-rank update. RandLoRA \cite{2025_ICLR_RandLoRA} employs parameter-efficient random matrix combinations to generate full-rank matrices. While these methods theoretically produce full rank updates, we observe their effective rank \cite{2007_ESPC_effectiverank} is usually low (i.e. the tail singular values are very small), suggesting they might not fully capture the desired spectral properties. 

\subsection{Motivations}
Our analysis of existing work indicates that the low-rank update in LoRA can limit performance on challenging tasks and with complex architectures. While recent advancements aim to overcome this by generating theoretically full-rank updates via random basis combinations, Kronecker products, or sine activations, we question whether this theoretical rank consistently translates to high effective rank. We hypothesize that the Khatri-Rao product offers an alternative method for constructing full-rank update matrices that demonstrably achieve higher effective rank. This increased effective rank would enable to learn more complex feature representations, leading to improved performance across diverse tasks and architectures.

\section{Khatri-Rao Adapters (KRAdapter)}

This section details the formulation of the proposed Khatri-Rao Adapters (KRAdapter) for parameter-efficient fine-tuning. We begin by briefly discussing Low-Rank Adaptation (LoRA) before introducing the KRAdapter mechanism and analyzing its parameter efficiency.

\subsection{Preliminaries}

Consider a pre-trained linear layer in a deep neural network, characterized by a weight matrix $\mathbf{W}_0 \in \mathbb{R}^{d_{out} \times d_{in}}$, where $d_{in}$ is the input dimension and $d_{out}$ is the output dimension. During standard full fine-tuning, the weight matrix $\mathbf{W}_0$ is updated directly by gradient descent. Given an input $\mathbf{x} \in \mathbb{R}^{d_{in}}$, the output $\mathbf{h}$ is computed as:

\begin{equation}
\mathbf{h} = \mathbf{W}_0 \mathbf{x} + \mathbf{b},
\end{equation}

where $\mathbf{b} \in \mathbb{R}^{d_{out}}$ is the bias term (for simplicity, we omit the bias term in subsequent derivations but it can be easily incorporated).

Low-Rank Adaptation (LoRA) \cite{2022_ICLR_lora} addresses the inefficiency of full fine-tuning by freezing the pre-trained weights $\mathbf{W}_0$ and introducing a low-rank update. Specifically, a low-rank matrix $\Delta \mathbf{W} = \mathbf{B} \mathbf{A}$ is added to the original weights, where $\mathbf{A} \in \mathbb{R}^{r \times d_{in}}$ and $\mathbf{B} \in \mathbb{R}^{d_{out} \times r}$, with $r \ll \min(d_{in}, d_{out})$ being the rank. During fine-tuning, only the parameters of $\mathbf{A}$ and $\mathbf{B}$ are updated. The output of the LoRA-adapted layer is:

\begin{equation}
\mathbf{h} = (\mathbf{W}_0 + \alpha \mathbf{B} \mathbf{A}) \mathbf{x},
\end{equation}
where $\alpha$ is a scaling factor that helps in stabilizing training and controlling the magnitude of the adapter.

\subsection{Method formulation}
In contrast to LoRAs, our proposed method, Khatri-Rao Adapters (KRAdapter), achieves parameter efficiency without imposing the low-rank constraint.
This is achieved by constructing $\Delta \mathbf{W}$ using the \textit{Khatri-Rao product}, otherwise known as a column-wise Kronecker product.

\textbf{Definition 3.1 (Khatri-Rao Product):} Given two matrices $\mathbf{U} \in \mathbb{R}^{a \times c}$ and $\mathbf{V} \in \mathbb{R}^{b \times c}$, their Khatri-Rao product, denoted by $\mathbf{U} \odot \mathbf{V}$, is a matrix of defined as:
\begin{equation}
\mathbf{U} \odot \mathbf{V} =
    \begin{bmatrix}
        u_{11} \mathbf{v}_1, u_{12} \mathbf{v}_2, ..., u_{1c} \mathbf{v}_c \\
        ... \\
        u_{a1} \mathbf{v}_1, u_{a2} \mathbf{v}_2, ..., u_{ac} \mathbf{v}_c \\
    \end{bmatrix}
    \in \mathbb{R}^{ab \times c},
\end{equation}
where $u_{ij}$ is the entry in the $i$-th row and $j$-th column of $\mathbf{U}$ and $\mathbf{v}_j \in \mathbb{R}^{b}$ is the $j$-th column of $\mathbf{V}$.

For a weight update $\Delta \mathbf{W}$ of size $d_{out} \times d_{in}$,
define $\mathbf{U} \in \mathbb{R}^{k_1 \times d_{in} }$ and $\mathbf{V} \in \mathbb{R}^{k_2 \times d_{in}}$,
where $k_1 k_2 = d_{out}$.

The forward pass of a linear layer with KRAdapter is as follows
\begin{equation}
\mathbf{h} = \left( \mathbf{W}_0 + \alpha \mathbf{U} \odot \mathbf{V} \right) \mathbf{x},
\end{equation}
where $\alpha$ is a scaling factor similar to LoRA that we typically set to $0.1$. 
Prior to training, $\mathbf{U}$ is initialized as zeros and $\mathbf{V}$ using a Kaiming uniform initialization~\cite{2015_ICCV_kaiminginit} with a negative slope coefficient of $\sqrt{1/k_1}$ which we empirically find to be the optimal for convergence.

\subsection{Parameter efficiency}
The number of trainable parameters introduced by KRAdapter is determined by the dimensions of the matrices $\mathbf{U}$ and $\mathbf{V}$, and the choice of $k_1$ (or $k_2$ resp.). The total number of trainable parameters for KRAdapter is:
\begin{equation}
N = d_{in}(k_1 + k_2). 
\end{equation}
$N$ is minimized when $k_1=k_2=\sqrt{d_{out}}$, which we adopt as the default configuration.
Details are deferred to Appendix~\ref{app:minimath}.
When $d_{out}$ is not a square number, we set $k_1 = \lfloor \sqrt{d_{out}}\rfloor$ and $k_2 = \lceil \frac{d_{out}}{k_1} \rceil$ such that $k_1 k_2 \ge d_{out}$.
The resultant product $\mathbf{U} \odot \mathbf{V}$ is then truncated to size $d_{out} \times d_{in}$.

KRAdapter achieves the same parameter reduction a LoRA of rank between 16 and 32 depending on matrix shapes, a common setting~\cite{2022_ICLR_lora,2024_ICML_DoRA,2023_ICLR_AdaLoRA} without incurring any significant increase in computational cost.

\subsection{Full-rank guarantee\label{sec:fullrank}}

This section establishes a full rank guarantee for the Khatri-Rao product when its constituent matrices are randomly initialized. This full rank property, as we will show, is crucial to the improvement of robustness against distribution shifts.

For brevity of exposition, let us assume $k_1=k_2=k $, i.e., $\mathbf{U} \in \mathbb{R}^{k \times d_{in}}$ and $\mathbf{V} \in \mathbb{R}^{k \times d_{in}}$. We assume that each matrix has rank $k$. Specifically, we consider the case where the entries of $\mathbf{U}$ and $\mathbf{V}$ are independently and identically distributed (i.i.d.) random variables drawn from either a standard Gaussian distribution $\mathcal{N}(0, 1)$ or a uniform distribution on $[-\delta, \delta]$ for some $\delta > 0$.

\begin{theorem}\label{theo:fullrank}
Let $\mathbf{U} \in \mathbb{R}^{k \times d_{in}}$ and $\mathbf{V} \in \mathbb{R}^{k \times d_{in}}$ where $k \leq d_{in} \leq k^2$ be matrices whose entries are chosen i.i.d. from a standard Gaussian or uniform distribution. Then, the Khatri-Rao product $\mathbf{U} \odot \mathbf{V}$ has full column rank almost surely, i.e.,
$$ \text{rank}(\mathbf{U} \odot \mathbf{V}) = d_{in}. $$
\end{theorem}

The detailed proof of Theorem~\ref{theo:fullrank} is provided in Appendix~\ref{app:theorank}.  Crucially, this theorem highlights a significant distinction from low-rank adaptation (LoRA) techniques. When employing LoRA updates to matrices $\mathbf{U}$ and $\mathbf{V}$, the resulting matrix, even after combination, can have a maximum rank of at most $k \leq d_{in}$, thus inherently producing rank-deficient matrices. In contrast, Theorem~\ref{theo:fullrank} demonstrates that the Khatri-Rao product of the same randomly initialized matrices $\mathbf{U}$ and $\mathbf{V}$ almost surely achieves full column rank, thereby providing a mechanism to construct high-rank matrices in a parameter-efficient manner. 

\subsection{Differences with Kronecker Adapters}
Since the Khatri-Rao product can be seen as column-wise Kronecker product,
our proposed KRAdapter and Kronecker Adapters (KronA)~\cite{2023_NeurIPSW_krona} bear some similarity.
However, due to the different constructions, the resulting properties differ significantly.
Specifically, we find that the Khatri-Rao product leads to different spectral properties in the resultant matrix, which displays significantly higher tail singular values. We find this to be a desirable property for PEFT methods as it 056
increases the representation power when modelling various types of weight matrices (Figure~\ref{fig:results-toy-relative}) and also improves OOD generalization and robustness in language and vision tasks (Table~\ref{tab:oodgen} and 3~\ref{tab:commsense})
Further discussion supported by empirical differences is available in Appendix~\ref{app:morediffkrona}.

\section{Experiments}\label{exp}
\subsection{Matrix approximation\label{exp:toy}}
To isolate and understand the inherent strengths and weaknesses of KRAdapter compared to LoRA and other full-rank PEFT methods, we design a series of controlled experiments using synthetic weight matrices with distinct structural properties. 

We design six matrix patterns to investigate parameter-efficient fine-tuning algorithms. The \textbf{normally distributed random matrix} acts as a high-rank baseline, evaluating general approximation performance. The \textbf{random 90\% sparse matrix} simulates scenarios where critical pre-trained parameters should remain unchanged. The \textbf{PCA-whitened random matrix} assesses algorithms' ability to handle highly decorrelated representations. The \textbf{low-rank random matrix}, constructed by keeping only the top quartile of singular values, tests full-rank algorithms' capacity to model low-rank structures. The \textbf{CLIP ImageNet fine-tuned weight delta (Vision or Language)}, obtained by the element-wise difference between the pre-trained CLIP-ViT-L/14 weights and the weights obtained after standard fine-tuning on ImageNet (also known as task vector~\cite{2024_NeurIPS_atlas}), represents a realistic target weight for LoRA in transformer in real-world fine-tuning. The \textbf{High/low frequency features}, generated using superposed sinusoidal functions with varying frequencies ($[1000, 10000]$ Hz and $[1, 100]$ Hz respectively), assess algorithms' bias towards feature frequencies.

\textbf{Methodology} For each target matrix pattern, we train the PEFT algorithms to minimize the mean squared error (MSE) between the estimated matrix and the target matrix.  We train on $1024\times 768$ matrices which is the dimension of the query part of the attention matrix in ViT-L/14. All algorithms are trained with the same or higher number of trainable parameters than KRAdapter to ensure a fair comparison. We report the absolute nuclear reconstruction error (average absolute element-wise singular value difference) to evaluate each algorithm's capacity to capture the spectrum of the target matrix. Our investigation focuses on comparing with LoRA~\cite{2022_ICLR_lora} and its full-rank alternatives, specifically: SinLoRA~\cite{2024_arxiv_sinelora}, RandLoRA~\cite{2025_ICLR_RandLoRA}, and Krona~\cite{2023_NeurIPSW_krona}. Further details are available in appendix~\ref{app:toy}.

\textbf{Results} We present results in Figure~\ref{fig:results-toy-relative} where our analysis yields several critical insights. First, full-rank PEFT algorithms consistently achieve lower approximation errors compared to LoRA across a range of target matrices (low rank matrices excepted). The advantage of high-rank PEFT methods is particularly apparent when approximating matrices with highly de-correlated features such as random, whitened or sparse noise. Results on approximating traditionally fine-tuned weights for the CLIP vision and language backbones are more nuanced as RandLoRA and Krona struggle to emulate the fine-tuned delta. This difference does not translate to poor performance in practice though which indicates these algorithms probably learn to solve tasks in a different manner than standard fine-tuning does. KRAdapter is especially performant at approximating the CLIP fine-tuned deltas as well as the high frequency matrix, demonstrating a high adaptability to learn highly tailored features. We evidence a key limitation of full-rank methods when approximating explicitly low-rank target matrices, as the algorithms did not demonstrate a substantial advantage over LoRA. Further analysis suggests a potential affinity of LoRA towards low-frequency feature representations as full-rank algorithms struggle to improve over LoRA's solution in this case. The low frequency matrix is however also low rank by nature which would favour LoRA's solution. Overall, KRAdapter provides a strong solution to all cases, low rank matrix excluded, which conforts us in its suitability as a parameter-efficient algorithm. Visualization of the approximated matrices is available in Figure~\ref{fig:vis-toy} in the appendix.

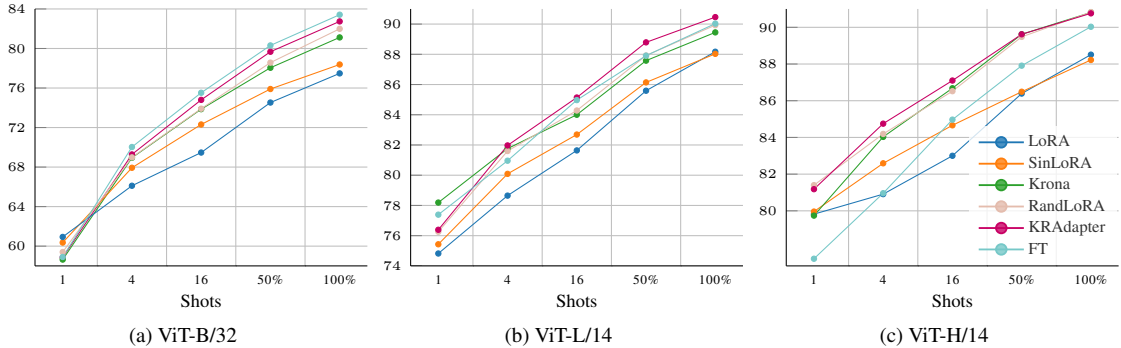
\begin{figure*}
    \centering
    \subfloat[ViT-B/32]{\input{figures/ViT-B-32}}   
    \subfloat[ViT-L/14]{\input{figures/ViT-L-14}}    
    \subfloat[ViT-H/14]{\input{figures/ViT-H-14}}
    \caption{Tuning CLIP's vision and language backbone for image classification. Accuracy (\%) averaged over 11 datasets.}
    \label{fig:results-clip-vision}
\end{figure*}

\begin{table*}[h]
\centering
\setlength{\tabcolsep}{2pt}
\small
\begin{tabular}{lcccccccccccc}
\toprule
Method & \multicolumn{4}{c}{ViT-B/32} & \multicolumn{4}{c}{ViT-L/14} & \multicolumn{4}{c}{ViT-H/14} \\
\cmidrule(r){1-1}\cmidrule(r){2-5}\cmidrule(r){6-9}\cmidrule(r){10-13}
& Nat. & Spe. & Struc. & Avg. & Nat. & Spe. & Struc. & Avg. & Nat. & Spe. & Struc. & Avg. \\
\cmidrule(r){1-1}\cmidrule(r){2-5}\cmidrule(r){6-9}\cmidrule(r){10-13}
LoRA & 70.6 & 82.6 & 45.2 & 66.1 & 82.3 & 87.0 & 54.7 & 74.6 & 84.1 & 86.7 & 56.8 & 75.9 \\
SinLoRA & 72.0 & 83.1 & 53.1 & 69.4 & 82.5 & 86.8 & 56.8 & 75.4  & 84.2 & 86.9 & 58.2 & 76.4  \\
RandLoRA  & 74.2 & 82.4 & 53.1 & 69.9 & 82.8 & 87.0 & \textbf{58.3} & 76.0 & 83.9 & \textbf{87.3} & \textbf{59.5} & 76.9 \\
Krona & 73.7 & 83.4 & 53.3 & 70.1 & 83.4 & \textbf{87.4} & 57.5 & 76.1 & 84.9 & 86.7 & 58.0 & 76.6 \\
KRAdapter & \textbf{76.0} & \textbf{84.0} & \textbf{53.3} & \textbf{71.1} & \textbf{84.8} & 86.9 & 57.2 & \textbf{76.3} & \textbf{85.8} & 87.0 & 58.1 & \textbf{77.0} \\
\midrule
FT & 75.7 & 83.5 & 53.1 & 70.8 & 83.2 & 87.4 & 58.3 & 76.3 & 85.4 & 84.8 & 43.7 & 71.3 \\
\bottomrule
\end{tabular}
\caption{Adaptation performance on VTAB1k's structured datasets}
\label{tab:vtab1k}
\end{table*}

\subsection{Vision-language tasks}
\subsubsection{Common datasets}
We now move on to real data and compare the performance of several full-rank PEFT algorithms when applied to fine-tuning CLIP models~\cite{2021_ICML_CLIP} (both vision and language encoders) across a diverse suite of eleven image classification datasets (detailed in Appendix~\ref{app:classifapp}). We evaluate Vision Transformer-based (ViT) architectures, initialized with publicly available openCLIP weights~\cite{2023_CVPR_openclip}. The PEFT algorithms are only trained on the attention heads. As an indicative, we also report the performance of full fine-tuning (FT), where all parameters of the network are updated. We tune the hyper-parameter of each method to achieve optimal performance.

Experiments are conducted in few-shot settings (1, 4, and 16 shots) and with varying dataset sizes (50\% and 100\% of the available training data). The training data is exactly the same for each algorithm to ensure fair comparisons. Results are visualized in Figure~\ref{fig:results-clip-vision} with detailed results in Table~\ref{tab:image_class} in the appendix.

Our results confirm the observations of~\cite{2025_ICLR_RandLoRA} about the importance of full-rank when fine-tuning vision-language models as these algorithms largely outperform LoRA with equal trainable parameters. We also confirm that full-rank PEFT algorithms can improve performance over standard fine-tuning by limiting over-fitting for larger models (ViT-H/14 especially). KRAdapter's enhanced representational capacity translates to a further performance improvements compared to other full-rank PEFT algorithms. This performance advantage is observed across most settings, although we observe a performance saturation for $\geq 50\%$ data settings for the larger ViT-H/14 architecture. KRAdapter's row-wise Kronecker formulation is additionally systematically superior to Krona's direct Kronecker approach.

\begin{table*}[htbp]
\centering
\setlength{\tabcolsep}{2pt}
\small
\begin{tabular}{lccccccccccccccccc}
\toprule
& \multicolumn{5}{c}{ViT-B/32} & \multicolumn{5}{c}{ViT-L/14} & \multicolumn{5}{c}{ViT-H/14} \\
\cmidrule(r){2-6}\cmidrule(l){7-11}\cmidrule(l){12-16}
& ID & OOD & $r_{gen}\uparrow$ & Nuc$\downarrow$ & Fro$\downarrow$ & ID & OOD & $r_{gen}\uparrow$ & Nuc$\downarrow$ & Fro$\downarrow$ & ID & OOD & $r_{gen}\uparrow$ & Nuc$\downarrow$ & Fro$\downarrow$ \\
\cmidrule(r){2-6}\cmidrule(l){7-11}\cmidrule(l){12-16}
Zero-shot & 62.6 & 55.5 & n/a & n/a & n/a & 75.4 & 74.1 & n/a & n/a & n/a & 77.9 & 75.7 & n/a & n/a & n/a \\
\midrule
LoRA & 72.9 & 58.2 & 0.27 & 19.6 & 4.1 & 83.3 & 77.2 & 0.39 & 46.2 & 6.6 & 83.7 & 75.9 & 0.04 & 70.0 & 7.4 \\
SinLoRA & 72.8 & 58.6 & 0.31 & 223.5 & 4.7 & 82.8 & 76.3 & 0.30 & 61.6 & 7.1 &  83.6 & 77.1 & 0.25 & 90.6 & 8.1 \\
RandLoRA & \textbf{73.0} & 58.4 & 0.28 & 46.2 & 5.7 & 82.8 & 76.2 & 0.29 & 159.9 & 11.8 & 82.9 & 75.4 & -0.05 & 508.3 & 21.2\\
Krona & 72.4 & 58.6 & 0.32 & 44.1 & 4.9 & \textbf{84.1} & 77.8 & 0.42 & 42.7 & 5.0 & \textbf{85.0} & 78.2 & 0.36 & 123.3 & 9.2 \\
KRAdapter & 72.5 & \textbf{59.9} & \textbf{0.45} & \textbf{7.3} & \textbf{2.3} & 83.6 & \textbf{78.2} & \textbf{0.49} & \textbf{9.8} & \textbf{2.8} & 84.7 & \textbf{78.8} & \textbf{0.48} & \textbf{32.8} & \textbf{5.5} \\
\midrule
FT & 75.5 & 58.9 & 0.26 & 0.7 & 0.8 & 85.1 & 76.7 & 0.27 & 1.1 & 1.0 & 85.5 & 78.6 & 0.39 & 4.3 & 1.9 \\
\bottomrule
\end{tabular}
\caption{Generalization to distribution shift. We report ID accuracy on ImageNet, OOD accuracy averaged over 7 OOD datasets, the $r_{gen}$ ratio and the nuclear and frobenius norm of the updates. The $\uparrow$ and $\downarrow$ indicates if higher or lower scores are better.}
\label{tab:oodgen}
\end{table*}

\subsubsection{Specialized datasets}
To further investigate performance on tasks necessitating strong model adaptation, we conduct experiments using the Visual Task Adaptation Benchmark (VTAB1k)~\cite{2019_arxiv_vtab1k} (detailed in Appendix~\ref{app:vtab1kdatasets}). VTAB1k contains 19 datasets grouped into 3 categories. The structured and specialized subsets as especially interesting to us as they contain tasks less likely to be encountered during CLIP's pre-training. The specialized subset focuses on predicting specialist medical or satellite imagery while the structured subset aims to predict object state attributes such as distance, location, count, and rotation. Prior work has shown that CLIP exhibits limitations on these types of tasks~\cite{2022_ACL_reclip}, making them an ideal benchmark for assessing the enhanced representation learning capacity of full-rank adaptation algorithms. The results are presented in Table~\ref{tab:vtab1k}, where we report average accuracy over the 3 subsets (detailed results are available in appendix~\ref{app:vtabdetail}). We observe that KRAdapter performs especially well on the natural and structured subsets but can struggle to generalize as well as RandLoRA to the structured subset for the larger models. Standard fine-tuning also struggles to generalize to the structured subset for ViT-H/14 which may indicate an implicit regularization in RandLoRA preventing overfitting, thus allowing to perform better in this specific setting. When averaged over the whole benchmark however, KRAdapter performs better across all architectures.

\begin{table*}[htbp]
\centering
\setlength{\tabcolsep}{3pt}
\small
\begin{tabular}{lcccccccccccc}
\toprule
Method & SIQA & ArcE & ArcC & OBQA & HellaS  & BoolQ & PIQA & WinoG & \textbf{ID} & \textbf{NID} & \textbf{OOD} & \textbf{Avg.} \\
\midrule
& \multicolumn{9}{c}{LLama3.1-8B} \\
\midrule
Zero-shot & 20.73 & 26.18 & 22.35 & 16.00 & 25.69 & 43.46 & 45.87 & 17.60 & 21.32 & 25.69 & 35.64 & 27.23 \\ 
LoRA & 79.07 & 87.12 & 74.66 & 83.40 & 53.70 & 42.35 & 76.61 & 49.80 & 81.06 & 53.70 & 55.62 & 68.34 \\
DoRA   & 79.79 & 88.22 & 77.56 & 84.40 & 52.16 & 41.77 & 79.87 & \textbf{57.38} & 82.49 & 52.16 & 57.80 & 70.14 \\
SinLoRA & 76.46 & \textbf{93.14} & \textbf{86.86} & \textbf{88.40} & 51.57 & 43.94 & \textbf{84.77} & 46.57 & \textbf{86.21} & 51.57 & 58.43 & 71.46 \\
RandLoRA & \textbf{79.79} & 87.67 & 76.45 & 82.60 & 57.22 & 52.51 & 79.22 & 54.38 & 81.63 & \textbf{57.22} & 60.83 & 71.23 \\
Krona  & 79.17 & 87.96 & 75.43 & 82.40 & 50.07 & \textbf{59.14} & 79.60 & 56.67 & 81.24 & 50.07 & 61.37 & 71.31 \\
KRAdapter  & 79.27 & 90.32 & 78.16 & 82.40 & \textbf{55.34} & 58.23 & 80.90 & 54.85 & 82.54 & 55.34 & \textbf{64.66} & \textbf{72.43} \\
\midrule
& \multicolumn{9}{c}{Qwen2.5-7B} \\
\midrule
Zero-shot & 5.02 & 5.18 & 4.10 & 4.80 & 1.79 & 31.44 & 4.73 & 27.70 & 4.77 & 1.79 & 21.29 & 10.59 \\
LoRA   & 51.13 & 92.13 & 85.84 & 88.40 & 71.01 & 43.55 & 83.79 & 40.81 & 79.37 & 71.01 & 59.79 & 69.58 \\
DoRA &  77.94 & 87.92 & 79.52 & 87.80 & 77.95 & 57.22 & 85.64 & \textbf{61.64} & 83.30 & 77.95 & 70.61 & 76.95 \\
SinLoRA  & 77.33 & 88.93 & 74.83 & 85.40 & 58.61 & 35.26 & 81.34 & 54.78 & 81.62 & 58.61 & 57.12 & 69.56 \\
RandLoRA  & \textbf{79.79} & \textbf{94.87} & 87.63 & 89.00 & 78.20 & 47.40 & \textbf{85.69} & 60.85 & 87.82 & 78.20 & 68.04 & 77.93 \\
Krona & 78.30 & 93.81 & \textbf{88.14} & \textbf{91.60} & \textbf{80.93} & 49.05 & 85.42 & 57.54 & \textbf{87.96} & \textbf{80.93} & 68.23 & 78.10 \\
KRAdapter & \textbf{79.79} & 94.61 & 86.95 & 89.80 & 75.31 & \textbf{62.14} & 84.66 & 61.01 & 87.79 & 75.31 & \textbf{70.78} & \textbf{79.28} \\
\bottomrule
\end{tabular}
\caption{4bit quantized LLama3.1-8B and Qwen2.5-7B on commonsense reasoning tasks. We finetune the Key and Value matrices of the attention layers. We evaluate on in-distribution (ID), near in-distribution (NID) and out-of-distribution (OOD) test sets and bold the best results.}
\label{tab:commsense}
\end{table*}

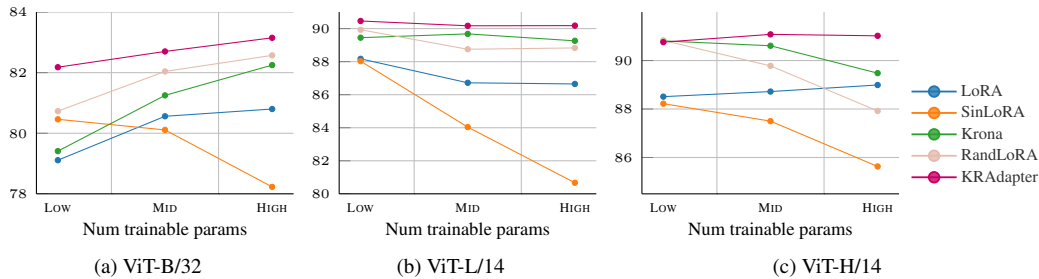
\begin{figure*}
    \centering
    \subfloat[ViT-B/32]{\input{figures/param_scaling_ViT-B-32}}   
    \subfloat[ViT-L/14]{\input{figures/param_scaling_ViT-L-14}}    
    \subfloat[ViT-H/14]{\input{figures/param_scaling_ViT-H-14}}
    \caption{Increasing the number of trainable parameters in PEFT algorithms. Accuracy (\%) averaged over 11 datasets.}
    \label{fig:results-scaling}
\end{figure*}

\subsubsection{Out-of-Distribution Generalization\label{sec:oodgen}}
We expect that the balanced singular value distribution of KRAdapter will allow for better out-of-distribution (OOD) robustness at test time. We thus propose to investigate whether parameter-efficient adaptations, optimized on in-distribution (ID) data, exhibit robustness and transferability to out-of-distribution (OOD) scenarios. We utilize a suite of established OOD benchmark datasets: ImageNet-A~\cite{2021_CVPR_ImageNetA}, ImageNet-Sketch~\cite{2019_NeurIPS_ImageNetSketch}, ImageNet-R~\cite{2021_ICCV_ImageNetR}, ImageNet-V2~\cite{2019_ICML_ImageNetV2}, and CIFAR-100~\cite{2009_CIFAR}. These datasets encompass diverse distribution shifts: adversarial, stylized, renditions and  domain shift, details in appendix~\ref{app:ooddatasets}. To quantify generalization, we introduce the ratio $r_{gen} = \frac{\Delta_{ood}}{\Delta_{id}}$, measuring the relative OOD accuracy gain ($\Delta_{ood}$, averaged over OOD datasets) to ID gain ($\Delta_{id}$, ImageNet validation) over the zero-shot model. Table~\ref{tab:oodgen} reports the results where we fine-tune the attentions layers of the vision backbone of CLIP only, maintaining frozen language embeddings to classify. KRAdapter consistently achieves a higher $r_{gen}$ than other PEFT methods. We hypothesize KRAdapter's spectral properties enable it to efficiently work in the parameter space without overfitting to a subset of directions, unlike rank-constrained methods prone to overfitting dominant ID features. To validate this hypothesis and quantify how much the algorithms strays from the pre-trained weights, we evaluate the average Nuclear and Frobenius norm of the learned updates over the attention layers. We find that KRAdapter produces updates with smaller norms than other PEFT algorithms which is in line with the very small norm of the traditionally fine-tuned model and has been previously reported to be beneficial for OOD generalization~\cite{2017_arxiv_spectralnormreg, 2025_ICLR_spectralregcontinual}. Table~\ref{tab:visionrank} in the appendix further reports the effective rank obtained when fine-tuning the vision-language architectures where we observe that KRAdapter systematically leads to higher effective ranks.

\subsection{Commonsense reasoning}
We further investigate resource-efficient fine-tuning of Large Language Models (LLMs) for commonsense reasoning through 4-bit quantization-aware training and evaluation. We evaluate this approach on Llama3.1-8B and Qwen2.5-7B, which are well suited to consumer grade GPUs.  
We fine-tune the key and value projection matrices within the attention layers, optimizing hyper-parameters for each algorithm to maximize performance. To specifically evaluate out-of-distribution (OOD) generalization, we slightly modify the commonsense reasoning benchmark established in previous research~\cite{2023_arXiv_llmadapters,2024_ICML_DoRA}. While the standard commonsense benchmark trains and tests on the same tasks with identical question and answer formats, we train on the multi-choice questions with "answer\{1\dots5\}" format (Science-QA, ARC, and OpenBookQA) and subsequently evaluate on: (1) in-distribution test sets from these training datasets; (2) near in-distribution HellaSwag test set (multi-choice, "ending\{1\dots4\}" format); and (3) out-of-distribution BoolQ, PiQA, and WinoGrande test sets (binary answers). This evaluation framework is designed to comprehensively assess the generalization capabilities of each PEFT algorithm, a crucial attribute for deploying robust LLMs across diverse real-world scenarios. Further details about our commonsense benchmark are available in appendix~\ref{app:commonsense}.

Table~\ref{tab:commsense} presents a comparative analysis of the PEFT methods. The results demonstrate a clear distinction: rank-restricted algorithms, such as LoRA, exhibit limited generalization to out-of-distribution (OOD) datasets. In contrast, full-rank algorithms significantly enhance the zero-shot model to answer these OOD tasks despite not specifically training for them. Notably, KRAdapter outperforms state-of-the-art methods across architectures. These findings support the conclusions of Section~\ref{sec:oodgen}, reinforcing evidence of the efficacy of KRAdapter in learning generalizable representations for language models.

\subsection{Ablation study}
We are interested here in challenging our design choices for KRAdapter and evaluating the impact of scaling the number of trainable parameters further. 
We study 3 scenario named low, medium and high where we iteratively double the amount of trainable parameters from the base (low) configuration used in other experiments. In KRAdapter, this is achieved by setting $k_1=\{d_{out}^{\frac{1}{2}},d_{out}^{\frac{1}{3}},d_{out}^{\frac{1}{4}}\}$. Ideally, performance should increase with every parameter increase to suggest favorable scaling laws. Results are reported in Figure~\ref{fig:results-scaling} for the vision-language task. We find that for smaller models, full-rank methods scale better with increases in parameter budgets whereas LoRA saturates faster. For larger models however, performance does not increase beyond the low configuration which ties in with the observed saturation in larger models of Figure~\ref{fig:results-clip-vision}. This aligns with findings of Albert~\etal~\cite{2025_ICLR_RandLoRA}. SinLoRA is a notable exception as more parameters leads to a decrease in accuracy. For ViT-H/14, KRAdapter is the only full rank PEFT algorithm to not drop in accuracy with increasing budget. This highlight the robustness of KRAdapter to over-fitting which helps it in achieving good OOD generalization. 

\section{Limitations}
Although we have shown KRAdapter is a good alternative to LoRA when high effective rank or out-of-distribution generalization is important, we discuss here some limitations of our approach. First KRAdapter in its most parameter-efficient configuration trains more parameter than the most efficient LoRA configurations do. In practice we find that KRAdatper is equivalent to a rank 16 to 32 LoRA, a commonly used configuration. Exploring low-rank approximations for the matrices $U$ and $V$ prior to the Khatri-Rao product could potentially mitigate this, though the theoretical implications for the derived full-rank properties in Section~\ref{sec:fullrank} require rigorous investigation. Second, empirical evidence from controlled matrix estimation experiments indicates that KRAdapter's performance may be suboptimal when estimating matrices of extremely low rank.  In scenarios demanding strict subspace constraints on weight updates, such as continual learning paradigms, LoRA might exhibit superior suitability. Finally, while KRAdapter consistently demonstrates robust out-of-distribution generalization and faster training convergence, we observe instances with larger models where RandLoRA achieves marginally superior performance on purely in-distribution test sets. It is important to note that KRAdapter's in-distribution performance remains highly competitive in these cases, and its advantages in generalization and training efficiency compared to RandLoRA generally outweigh this minor performance delta.

\section{Conclusion}
This paper introduces Khatri-Rao Adapters (KRAdapter), a novel parameter-efficient fine-tuning technique designed to overcome the representation limitations of low-rank methods. KRAdapter generate full-rank update matrices with demonstrably higher effective rank, enabling to capture more complex feature interactions and improving OOD generalization by staying close to the pre-trained weights. KRAdapter achieves all this without sacrificing parameter efficiency and training time.
Our comprehensive experimental evaluation across vision-language, language understanding, and commonsense reasoning tasks consistently demonstrates KRAdapter produces improved results compared to LoRA and other full-rank PEFT alternatives, especially for specialized VTAB1k tasks that demand nuanced feature adaptation and for larger models that are prone to over-fitting. We specifically highlight that KRAdapter's improved representation space allows to produce representations closer to the zero-shot model that lead to better generalization under distribution shifts compared to other PEFT algorithms.
Crucially, KRAdapter maintains the computational advantages of LoRA in terms of training speed and memory footprint, offering a practical and effective solution for fine-tuning billion-scale models as demonstrated in the commonsense experiments. A notable limitation of KRAdapter, however, is that the minimum number of trainable parameter is larger than the most efficient LoRA configurations which currently renders it unadapted to extreme scarcity scenarios. Future work could investigate replacing $U$ and $V$ with low-rank formulations to achieve better efficiency. We however anticipate this formulation would change the efficient rank guarantees of KRAdapter.
Our findings underscore the importance of high effective rank weight updates in parameter-efficient adaptation of pre-trained models across diverse downstream applications. 

%% file: figures/ViT-B-32.tex
\begin{tikzpicture}
    \begin{axis}[
        width=6cm,
        height=5cm,
        font=\tiny,
        ymin=58,
        ymax=84,
        legend style={draw=none,at={(0.55,1.3)}, 
                    text width=1.3cm,
                    anchor=north,
                    legend columns=6,
                    fill=none,
                nodes={scale=0.8, transform shape},},
        xlabel={\scriptsize{Shots}},
        ylabel={},
        ymajorgrids=true,
        ytick={56,60,64,68,72,76,80,84},
        yticklabel={$\tiny\pgfmathprintnumber{\tick}$},
        ylabel shift=-5pt,
        xlabel shift=-3pt,
        axis x line*=bottom,
        axis y line*=left,  
        symbolic x coords={1,4,16,50\%,100\%},
        cycle list={color1,color2,color3,color4,color5,color6,color7,color8,color9,color21},
        xtick={1,4,16,50\%,100\%},
        xticklabel={\textsc{\tick}},
        minor x tick num=1,
        xminorgrids,
        minor tick length=0,
        major x tick style = transparent,
        mark size=1.0pt,
        ]    

\addplot+[mark=*,mark size=1pt,mark options={fill=color5},color=color5,draw opacity=1] coordinates {(1, 58.86)(4, 69.28)(16, 74.80)(50\%, 79.67)(100\%, 82.74)};

\addplot+[mark=*,mark size=1pt,mark options={fill=color3},color=color3,draw opacity=1] coordinates {(1, 58.64)(4, 68.94)(16, 73.86)(50\%, 78.05)(100\%, 81.12)};

\addplot+[mark=*,mark size=1pt,mark options={fill=color2},color=color2,draw opacity=1] coordinates {(1, 60.36)(4, 67.93)(16, 72.31)(50\%, 75.90)(100\%, 78.38)};

\addplot+[mark=*,mark size=1pt,mark options={fill=color4},color=color4,draw opacity=1] coordinates {(1, 59.40)(4, 68.98)(16, 73.91)(50\%, 78.57)(100\%, 81.99)};

\addplot+[mark=*,mark size=1pt,mark options={fill=color1},color=color1,draw opacity=1] coordinates {(1, 60.93)(4, 66.11)(16, 69.47)(50\%, 74.53)(100\%, 77.48)};

\addplot+[mark=*,mark size=1pt,mark options={fill=color6},color=color6,draw opacity=1] coordinates {(1, 58.90)(4, 70.03)(16, 75.52)(50\%, 80.31)(100\%, 83.42)};

\end{axis}
    \begin{axis}[
       xmin=0,
       xmax=16,
       ymin=1,
       ymax=2,
       hide axis,
       width=5cm,
       height=5cm,
       font=\small,
       mark size=1.0pt,
       legend style={
               at={(1.0,0.0)},
               anchor=south,
               draw=none,
               legend columns=1,
               fill opacity=0.8,
               nodes={scale=0.7, transform shape},
               cells={align=left},
               opacity=0, 
           },
       legend cell align={left},
   ]

\addplot+ [mark=*, mark size=2pt, mark options={fill=color1},color=color1, line width=0.7pt,solid] coordinates { (0,0) };
\addlegendentry{LoRA}

\addplot+ [mark=*, mark size=2pt, mark options={fill=color2},color=color2, line width=0.7pt,solid] coordinates { (0,0) };
\addlegendentry{SinLoRA}

\addplot+ [mark=*, mark size=2pt, mark options={fill=color3},color=color3, line width=0.7pt,solid] coordinates { (0,0) };
\addlegendentry{Krona}

\addplot+ [mark=*, mark size=2pt, mark options={fill=color4},color=color4, line width=0.7pt,solid] coordinates { (0,0) };
\addlegendentry{RandLoRA}

\addplot+ [mark=*, mark size=2pt, mark options={fill=color5},color=color5, line width=0.7pt,solid] coordinates { (0,0) };
\addlegendentry{KRAdapter}

\addplot+ [mark=*, mark size=2pt, mark options={fill=color6},color=color6, line width=0.7pt,solid] coordinates { (0,0) };
\addlegendentry{FT}

\end{axis}
\end{tikzpicture}

%% file: figures/ViT-L-14.tex
\begin{tikzpicture}
    \begin{axis}[
        width=6cm,
        height=5cm,
        font=\tiny,
        ymin=74,
        ymax=91,
        legend style={draw=none,at={(0.55,1.3)}, 
                    text width=1.3cm,
                    anchor=north,
                    legend columns=6,
                    fill=none,
                nodes={scale=0.8, transform shape},},
        xlabel={\scriptsize{Shots}},
        ylabel={},
        ymajorgrids=true,
        ytick={52,56,60,64,68,74,76,78,80,82,84,86,88,90},
        yticklabel={$\tiny\pgfmathprintnumber{\tick}$},
        ylabel shift=-5pt,
        xlabel shift=-3pt,
        axis x line*=bottom,
        axis y line*=left,  
        symbolic x coords={1,4,16,50\%,100\%},
        cycle list={color1,color2,color3,color4,color5,color6,color7,color8,color9,color21},
        xtick={1,4,16,50\%,100\%},
        xticklabel={\textsc{\tick}},
        minor x tick num=1,
        xminorgrids,
        minor tick length=0,
        major x tick style = transparent,
        mark size=1.0pt,
        ]    

    \addplot+[mark=*,mark size=1pt,mark options={fill=color1},color=color1,draw opacity=1] coordinates {(1, 74.82)(4, 78.65)(16, 81.64)(50\%, 85.59)(100\%, 88.17)};

    \addplot+[mark=*,mark size=1pt,mark options={fill=color2},color=color2,draw opacity=1] coordinates {(1, 75.43)(4, 80.09)(16, 82.69)(50\%, 86.14)(100\%, 88.03)};

    \addplot+[mark=*,mark size=1pt,mark options={fill=color3},color=color3,draw opacity=1] coordinates {(1, 78.19)(4, 81.76)(16, 84.00)(50\%, 87.58)(100\%, 89.45)};

    \addplot+[mark=*,mark size=1pt,mark options={fill=color4},color=color4,draw opacity=1] coordinates {(1, 76.26)(4, 81.60)(16, 84.28)(50\%, 87.92)(100\%, 89.93)};
    
    \addplot+[mark=*,mark size=1pt,mark     options={fill=color5},color=color5,draw opacity=1] coordinates {(1, 76.39)(4, 81.97)(16, 85.14)(50\%, 88.79)(100\%, 90.46)};

    \addplot+[mark=*,mark size=1pt,mark options={fill=color6},color=color6,draw opacity=1] coordinates {(1, 77.39)(4, 80.96)(16, 84.97)(50\%, 87.91)(100\%, 90.03)};

\end{axis}
    \begin{axis}[
       xmin=0,
       xmax=16,
       ymin=1,
       ymax=2,
       hide axis,
       width=5cm,
       height=5cm,
       font=\small,
       mark size=1.0pt,
       legend style={
               at={(1.0,0.0)},
               anchor=south,
               draw=none,
               legend columns=1,
               fill opacity=0.8,
               nodes={scale=0.7, transform shape},
               cells={align=left},       
               opacity=0, 
           },
       legend cell align={left},
   ]

\addplot+ [mark=*, mark size=2pt, mark options={fill=color1},color=color1, line width=0.7pt,solid] coordinates { (0,0) };
\addlegendentry{LoRA}

\addplot+ [mark=*, mark size=2pt, mark options={fill=color2},color=color2, line width=0.7pt,solid] coordinates { (0,0) };
\addlegendentry{SinLoRA}

\addplot+ [mark=*, mark size=2pt, mark options={fill=color3},color=color3, line width=0.7pt,solid] coordinates { (0,0) };
\addlegendentry{Krona}

\addplot+ [mark=*, mark size=2pt, mark options={fill=color4},color=color4, line width=0.7pt,solid] coordinates { (0,0) };
\addlegendentry{RandLoRA}

\addplot+ [mark=*, mark size=2pt, mark options={fill=color5},color=color5, line width=0.7pt,solid] coordinates { (0,0) };
\addlegendentry{KRAdapter}

\addplot+ [mark=*, mark size=2pt, mark options={fill=color6},color=color6, line width=0.7pt,solid] coordinates { (0,0) };
\addlegendentry{FT}

\end{axis}
\end{tikzpicture}

%% file: figures/ViT-H-14.tex
\begin{tikzpicture}
    \begin{axis}[
        width=6cm,
        height=5cm,
        font=\tiny,
        ymin=77,
        ymax=91,
        legend style={draw=none,at={(0.55,1.3)}, 
                    text width=1.3cm,
                    anchor=north,
                    legend columns=6,
                    fill=none,
                nodes={scale=0.8, transform shape},},
        xlabel={\scriptsize{Shots}},
        ylabel={},
        ymajorgrids=true,
        ytick={52,56,60,64,68,72,76,80,82,84,86,88,90,92},
        yticklabel={$\tiny\pgfmathprintnumber{\tick}$},
        ylabel shift=-5pt,
        xlabel shift=-3pt,
        axis x line*=bottom,
        axis y line*=left,  
        symbolic x coords={1,4,16,50\%,100\%},
        cycle list={color1,color2,color3,color4,color5,color6,color7,color8,color9,color21},
        xtick={1,4,16,50\%,100\%},
        xticklabel={\textsc{\tick}},
        minor x tick num=1,
        xminorgrids,
        minor tick length=0,
        major x tick style = transparent,
        mark size=1.0pt,
        ]

\addplot+[mark=*,mark size=1pt,mark options={fill=color1},color=color1,draw opacity=1] coordinates {(1, 79.82)(4, 80.91)(16, 83.00)(50\%, 86.39)(100\%, 88.51)};

\addplot+[mark=*,mark size=1pt,mark options={fill=color2},color=color2,draw opacity=1] coordinates {(1, 79.96)(4, 82.59)(16, 84.66)(50\%, 86.49)(100\%, 88.22)};

\addplot+[mark=*,mark size=1pt,mark options={fill=color3},color=color3,draw opacity=1] coordinates {(1, 79.74)(4, 84.03)(16, 86.68)(50\%, 89.62)(100\%, 90.81)};

\addplot+[mark=*,mark size=1pt,mark options={fill=color4},color=color4,draw opacity=1] coordinates {(1, 81.40)(4, 84.19)(16, 86.52)(50\%, 89.48)(100\%, 90.83)};

\addplot+[mark=*,mark size=1pt,mark options={fill=color5},color=color5,draw opacity=1] coordinates {(1, 81.18)(4, 84.75)(16, 87.10)(50\%, 89.62)(100\%, 90.76)};

\addplot+[mark=*,mark size=1pt,mark options={fill=color6},color=color6,draw opacity=1] coordinates {(1, 77.39)(4, 80.96)(16, 84.97)(50\%, 87.91)(100\%, 90.03)};

\end{axis}
    \begin{axis}[
       xmin=0,
       xmax=16,
       ymin=1,
       ymax=2,
       hide axis,
       width=5cm,
       height=5cm,
       font=\small,
       mark size=1.0pt,
       legend style={
               at={(1,0.0)},
               anchor=south,
               draw=none,
               legend columns=1,
               fill opacity=0.8,
               nodes={scale=0.7, transform shape},
               cells={align=left},
           },
       legend cell align={left},
   ]

\addplot+ [mark=*, mark size=2pt, mark options={fill=color1},color=color1, line width=0.7pt,solid] coordinates { (0,0) };
\addlegendentry{LoRA}

\addplot+ [mark=*, mark size=2pt, mark options={fill=color2},color=color2, line width=0.7pt,solid] coordinates { (0,0) };
\addlegendentry{SinLoRA}

\addplot+ [mark=*, mark size=2pt, mark options={fill=color3},color=color3, line width=0.7pt,solid] coordinates { (0,0) };
\addlegendentry{Krona}

\addplot+ [mark=*, mark size=2pt, mark options={fill=color4},color=color4, line width=0.7pt,solid] coordinates { (0,0) };
\addlegendentry{RandLoRA}

\addplot+ [mark=*, mark size=2pt, mark options={fill=color5},color=color5, line width=0.7pt,solid] coordinates { (0,0) };
\addlegendentry{KRAdapter}

\addplot+ [mark=*, mark size=2pt, mark options={fill=color6},color=color6, line width=0.7pt,solid] coordinates { (0,0) };
\addlegendentry{FT}

\end{axis}
\end{tikzpicture}

%% file: figures/param_scaling_ViT-B-32.tex
\begin{tikzpicture}
    \begin{axis}[
        width=5cm,
        height=4cm,
        font=\tiny,
        ymin=78,
        ymax=84,
        legend style={draw=none,at={(0.55,1.3)}, 
                    text width=1.3cm,
                    anchor=north,
                    legend columns=6,
                    fill=none,
                nodes={scale=0.8, transform shape},},
        xlabel={\scriptsize{Num trainable params}},
        ylabel={},
        ymajorgrids=true,
        ytick={52,56,60,64,68,74,76,78,80,82,84,86,88,90},
        yticklabel={$\tiny\pgfmathprintnumber{\tick}$},
        ylabel shift=-5pt,
        xlabel shift=-3pt,
        axis x line*=bottom,
        axis y line*=left,  
        symbolic x coords={Low, Mid, High},
        cycle list={color1,color2,color3,color4,color5,color6,color7,color8,color9,color21},
        xtick={Low, Mid, High},
        xticklabel={\textsc{\tick}},
        minor x tick num=1,
        xminorgrids,
        minor tick length=0,
        major x tick style = transparent,
        mark size=1.0pt,
        ]

    \addplot+[mark=*,mark size=1pt,mark options={fill=color1},color=color1,draw opacity=1] coordinates {(Low, 79.11)(Mid, 80.56)(High, 80.80)};
    
    \addplot+[mark=*,mark size=1pt,mark options={fill=color2},color=color2,draw opacity=1] coordinates {(Low, 80.46)(Mid, 80.11)(High, 78.23)};

    \addplot+[mark=*,mark size=1pt,mark options={fill=color3},color=color3,draw opacity=1] coordinates {(Low, 79.41)(Mid, 81.25)(High, 82.25)};

    \addplot+[mark=*,mark size=1pt,mark options={fill=color4},color=color4,draw opacity=1] coordinates {(Low, 80.73)(Mid, 82.04)(High, 82.57)};

    \addplot+[mark=*,mark size=1pt,mark options={fill=color5},color=color5,draw opacity=1] coordinates {(Low, 82.18)(Mid, 82.70)(High, 83.15)};

\end{axis}
    \begin{axis}[
       xmin=0,
       xmax=16,
       ymin=1,
       ymax=2,
       hide axis,
       width=5cm,
       height=5cm,
       font=\small,
       mark size=1.0pt,
       legend style={
               at={(0.7,0.0)},
               anchor=south,
               draw=none,
               legend columns=1,
               fill opacity=0.8,
               nodes={scale=0.7, transform shape},
               cells={align=left},
                opacity=0, 
           },
       legend cell align={left},
   ]

\addplot+ [mark=*, mark size=2pt, mark options={fill=color1},color=color1, line width=0.7pt,solid] coordinates { (0,0) };
\addlegendentry{LoRA}

\addplot+ [mark=*, mark size=2pt, mark options={fill=color2},color=color2, line width=0.7pt,solid] coordinates { (0,0) };
\addlegendentry{SinLoRA}

\addplot+ [mark=*, mark size=2pt, mark options={fill=color3},color=color3, line width=0.7pt,solid] coordinates { (0,0) };
\addlegendentry{Krona}

\addplot+ [mark=*, mark size=2pt, mark options={fill=color4},color=color4, line width=0.7pt,solid] coordinates { (0,0) };
\addlegendentry{RandLoRA}

\addplot+ [mark=*, mark size=2pt, mark options={fill=color5},color=color5, line width=0.7pt,solid] coordinates { (0,0) };
\addlegendentry{KRAdapter}

\end{axis}
\end{tikzpicture}

%% file: figures/param_scaling_ViT-L-14.tex
\begin{tikzpicture}
    \begin{axis}[
        width=5cm,
        height=4cm,
        font=\tiny,
        ymin=80.0,
        ymax=91,
        legend style={draw=none,at={(0.55,1.3)}, 
                    text width=1.3cm,
                    anchor=north,
                    legend columns=6,
                    fill=none,
                nodes={scale=0.8, transform shape},},
        xlabel={\scriptsize{Num trainable params}},
        ylabel={},
        ymajorgrids=true,
        ytick={52,56,60,64,68,74,76,78,80,82,84,86,88,90},
        yticklabel={$\tiny\pgfmathprintnumber{\tick}$},
        ylabel shift=-5pt,
        xlabel shift=-3pt,
        axis x line*=bottom,
        axis y line*=left,  
        symbolic x coords={Low, Mid, High},
        cycle list={color1,color2,color3,color4,color5,color6,color7,color8,color9,color21},
        xtick={Low, Mid, High},
        xticklabel={\textsc{\tick}},
        minor x tick num=1,
        xminorgrids,
        minor tick length=0,
        major x tick style = transparent,
        mark size=1.0pt,
        ]

    \addplot+[mark=*,mark size=1pt,mark options={fill=color1},color=color1,draw opacity=1] coordinates {(Low, 88.17)(Mid, 86.72)(High, 86.65)};    
    
    \addplot+[mark=*,mark size=1pt,mark options={fill=color2},color=color2,draw opacity=1] coordinates {(Low, 88.03)(Mid, 84.04)(High, 80.67)};

    \addplot+[mark=*,mark size=1pt,mark options={fill=color3},color=color3,draw opacity=1] coordinates {(Low, 89.45)(Mid, 89.68)(High, 89.26)};

    \addplot+[mark=*,mark size=1pt,mark options={fill=color4},color=color4,draw opacity=1] coordinates {(Low, 89.93)(Mid, 88.75)(High, 88.83)};
    
    \addplot+[mark=*,mark size=1pt,mark     options={fill=color5},color=color5,draw opacity=1] coordinates {(Low, 90.46)(Mid, 90.17)(High, 90.18)};

\end{axis}
    \begin{axis}[
       xmin=0,
       xmax=16,
       ymin=1,
       ymax=2,
       hide axis,
       width=5cm,
       height=5cm,
       font=\small,
       mark size=1.0pt,
       legend style={
               at={(0.7,0.0)},
               anchor=south,
               draw=none,
               legend columns=1,
               fill opacity=0.8,
               nodes={scale=0.7, transform shape},
               cells={align=left},
                opacity=0, 
           },
       legend cell align={left},
   ]

\addplot+ [mark=*, mark size=2pt, mark options={fill=color1},color=color1, line width=0.7pt,solid] coordinates { (0,0) };
\addlegendentry{LoRA}

\addplot+ [mark=*, mark size=2pt, mark options={fill=color2},color=color2, line width=0.7pt,solid] coordinates { (0,0) };
\addlegendentry{SinLoRA}

\addplot+ [mark=*, mark size=2pt, mark options={fill=color3},color=color3, line width=0.7pt,solid] coordinates { (0,0) };
\addlegendentry{Krona}

\addplot+ [mark=*, mark size=2pt, mark options={fill=color4},color=color4, line width=0.7pt,solid] coordinates { (0,0) };
\addlegendentry{RandLoRA}

\addplot+ [mark=*, mark size=2pt, mark options={fill=color5},color=color5, line width=0.7pt,solid] coordinates { (0,0) };
\addlegendentry{KRAdapter}

\end{axis}
\end{tikzpicture}

%% file: figures/param_scaling_ViT-H-14.tex
\begin{tikzpicture}
    \begin{axis}[
        width=5cm,
        height=4cm,
        font=\tiny,
        ymin=84.5,
        ymax=92,
        legend style={draw=none,at={(0.55,1.3)}, 
                    text width=1.3cm,
                    anchor=north,
                    legend columns=6,
                    fill=none,
                nodes={scale=0.8, transform shape},},
        xlabel={\scriptsize{Num trainable params}},
        ylabel={},
        ymajorgrids=true,
        ytick={52,56,60,64,68,74,76,78,80,82,84,86,88,90},
        yticklabel={$\tiny\pgfmathprintnumber{\tick}$},
        ylabel shift=-5pt,
        xlabel shift=-3pt,
        axis x line*=bottom,
        axis y line*=left,  
        symbolic x coords={Low, Mid, High},
        cycle list={color1,color2,color3,color4,color5,color6,color7,color8,color9,color21},
        xtick={Low, Mid, High},
        xticklabel={\textsc{\tick}},
        minor x tick num=1,
        xminorgrids,
        minor tick length=0,
        major x tick style = transparent,
        mark size=1.0pt,
        ]

\addplot+[mark=*,mark size=1pt,mark options={fill=color1},color=color1,draw opacity=1] coordinates {(Low, 88.51)(Mid, 88.72)(High, 88.99)};

\addplot+[mark=*,mark size=1pt,mark options={fill=color2},color=color2,draw opacity=1] coordinates {(Low, 88.22)(Mid, 87.50)(High, 85.63)};

\addplot+[mark=*,mark size=1pt,mark options={fill=color3},color=color3,draw opacity =1] coordinates {(Low, 90.81)(Mid, 90.61)(High, 89.48)};

\addplot+[mark=*,mark size=1pt,mark options={fill=color4},color=color4,draw opacity=1] coordinates {(Low, 90.83)(Mid, 89.78)(High, 87.92)};

\addplot+[mark=*,mark size=1pt,mark options={fill=color5},color=color5,draw opacity=1] coordinates {(Low, 90.76)(Mid, 91.08)(High, 91.02)};

\end{axis}
    \begin{axis}[
       xmin=0,
       xmax=16,
       ymin=1,
       ymax=2,
       hide axis,
       width=5cm,
       height=5cm,
       font=\small,
       mark size=1.0pt,
       legend style={
               at={(1.3,0.0)},
               anchor=south,
               draw=none,
               legend columns=1,
               fill opacity=0.8,
               nodes={scale=0.7, transform shape},
               cells={align=left},
           },
       legend cell align={left},
   ]

\addplot+ [mark=*, mark size=2pt, mark options={fill=color1},color=color1, line width=0.7pt,solid] coordinates { (0,0) };
\addlegendentry{LoRA}

\addplot+ [mark=*, mark size=2pt, mark options={fill=color2},color=color2, line width=0.7pt,solid] coordinates { (0,0) };
\addlegendentry{SinLoRA}

\addplot+ [mark=*, mark size=2pt, mark options={fill=color3},color=color3, line width=0.7pt,solid] coordinates { (0,0) };
\addlegendentry{Krona}

\addplot+ [mark=*, mark size=2pt, mark options={fill=color4},color=color4, line width=0.7pt,solid] coordinates { (0,0) };
\addlegendentry{RandLoRA}

\addplot+ [mark=*, mark size=2pt, mark options={fill=color5},color=color5, line width=0.7pt,solid] coordinates { (0,0) };
\addlegendentry{KRAdapter}

\end{axis}
\end{tikzpicture}

%% file: sec/3_appendix.tex
\newpage
\appendix
\onecolumn

\section{Mathematical supplement\label{app:theorems}}
\subsection{Minimizing the trainable parameters in KRAdapter\label{app:minimath}}

\paragraph{Khatri-Rao products:} Our main theorem is that an arbitrary matrix $W$ of size 
$m \times n$ can be approximated by the Khatri-Rao product of two matrices. In order to see how to do this we will need the column-wise vectorization operator $\mathbf{vec}$. 

\begin{definition}
Let $\mathbf{vec}$ denote the column-wise vectorization operator defined as follows. Given a matrix $A = [A_1\cdots A_n]$ of size $m \times n$, where each $A_i$ has size $m \times 1$, we define $\mathbf{vec}(A)$ to be the matrix of size $mn \times 1$ defined by
\begin{equation}
\mathbf{vec}(A) = \begin{bmatrix}
A_1 \\
\vdots \\
A_n
\end{bmatrix}    
\end{equation}
\end{definition}

\begin{theorem}\label{thm:KR_decomp}
Let $W$ be a matrix of size $m \times n$ such that 
$rank(W) = r$. Then there exists matrices $\overline{U}$ of size $m \times r$ and $\overline{V}$ of size $n \times r$ and a vector $\sigma$ of size $r \times 1$ such that
\begin{equation}
\mathbf{vec}(W) = (\overline{V} \odot \overline{U})\sigma. 
\end{equation}
\end{theorem}

\begin{proof}
We apply the SVD to $W$ to obtain the decomposition 
$W = U S V^T$. We can then write
\begin{equation}\label{eqn:svd_decomp}
    W = \sum_{i=1}^ru_iv_i^T\sigma_i  
\end{equation}
where $u_i$ is the ith column of $U$, $v_i$ is the ith column of $V$ and $\sigma_i$ is the ith singular value of $S$.
We then observe that if we vectorize \eqref{eqn:svd_decomp} we obtain
\begin{equation}
    \mathbf{vec}(W) = \sum_{i=1}^r\mathbf{vec}(u_iv_i^T)\sigma_i.
\end{equation}
The we note that since $u_i$ and $v_i$ are column vectors we have
\begin{equation}
\mathbf{vec}(u_iv_i^T) = v_i \otimes u_i.
\end{equation}
This gives
\begin{equation}
\mathbf{vec}(W) = \sum_{i=1}^r(v_i \otimes u_i)\sigma_i.
\end{equation}
If we define $\overline{U} = [u_1\cdots u_r]$ and 
$\overline{V} = [v_1\cdots v_r]$ then by definition of the Khatri-Rao product we have
\begin{equation}
\mathbf{vec}(W) = (\overline{V} \cdot \overline{U})\sigma     
\end{equation}
where $\sigma = (\sigma_1\cdots \sigma_r)^T$.
\end{proof}

Theorem \ref{thm:KR_decomp} shows that we can use Khatri-Rao products to approximate matrices. The importance of this approximation is that if were to use Khatri-Rao products for weights of a neural model, we get a parameter efficient decomposition of the weight matrix. 

In general, we can apply Khatri-Rao products to approximate weight matrices as follows:
Given a pretrained base weight of shape 
$d_{out} \times d_{in}$ we can take two matrices $U$ and $V$ of shapes
$k_1 \times d_{in}$ and $k_2 \times d_{in}$ respectively and consider the Khatri-Rao product $U \odot V$ of shape $k_1k_2 \times d_{in}$ where $k_1$, $k_2 > 0$. 
In order to get the shape right we need to take $k_2 = \frac{d_{out}}{k_1}$. Then the total number of parameters will be $(k_1 + \frac{d_{out}}{k_1})d_{in}$. This is minimized when $k_1 = \sqrt{d_{out}}$ so that 
$k_2 = \sqrt{d_{out}}$, which follows from the following lemma.
\begin{lemma}
   Let $f(x) = (\frac{m}{x} + x)n$ for $x > 0$ where $m$, $n > 0$. Then $f$ has a minimum at the point $x = \sqrt{m}$. 
\end{lemma}
\begin{proof}
Differentiating we see that $f'(x) = n + -\frac{mn}{x^2}$. Setting this to zero to find critical points gives
\begin{equation}
    n + -\frac{mn}{x^2} = 0 \Rightarrow \frac{m}{x^2} = 1
\end{equation}
which gives $x = \pm\sqrt{m}$. Since we are assuming $x > 0$ we have that $x = \sqrt{m}$ is a critical point. To understand what type of critical point this is we take the double derivative and find $f''(x) = \frac{2mn}{x^3}$. We then have that 
\begin{equation}
    f''(\sqrt{m}) = \frac{2mn}{m^{3/2}} = \frac{2n}{\sqrt{m}} > 0.
\end{equation}
This tells us the critical point $x = \sqrt{m}$ is a minimum point.
\end{proof}

If $d_{out}$ is not a perfect square we can take 
$k_1 = \lfloor d_{out} \rfloor$. Then $U$ has size $\sqrt{d_{out}} \times d_{in}$ and $V$ has shape $\sqrt{d_{out}} \times d_{in}$. The total parameters are
$2\sqrt{d_{out}}d_{in}$ which is much smaller than $d_{out}d_{in}$. We thus see that 
by using a Khatri-Rao product we obtain a low parameter approximation for the adaptors that have parameters in 
$\mathcal{O}(\sqrt{d_{out}}d_{in})$ which is much less than $\mathcal{O}(d_{out}d_{in})$ when $d_{out}$ and $d_{in}$ are large.

To further enhance parameter efficiency, we typically choose $d_{in}$ to be the smaller dimension of the original weight matrix $\mathbf{W}_0$. If $d_{out} < d_{in}$, we then transpose the resulting update to be applied to $\mathbf{W}_0$.

\subsection{Proving the Khatri-Rao of two random matrices is full rank\label{app:theorank}}
We can compare the construction of a matrix $\mathbf{W} \in  \mathbb{R}^{d_{in}\times d_{in}}$ from $\mathbf{U} \in \mathbb{R}^{k \times d_{in}}$ and $\mathbf{V} \in \mathbb{R}^{k \times d_{in}}$  where $k << d_{in}$ using a Khatri-Rao product compared to standard low rank approximations used in models such as LoRA. In the context of LoRA the matrices $U$ and $V$ would be multiplied as follows
$U^TV$ to produce a matrix of size $d_{in} \times d_{in}$. Since 
$k < d_{in}$ and assuming $U$ and $V$ have rank 
$k$, we then have by properties of the rank of a product that
\begin{equation}
Rank(U^TV) = k.    
\end{equation}
In particular there are no conditions we can impose on $U$ and $V$ such that $U^TV$ has rank greater than $k$. However, by taking a Khatri-Rao product we will show that under suitable conditions we can obtain a matrix with much larger rank $= \textit{min}(k^2, d_{in})$. To show this we need the following lemma borrowed from from Albert~\etal \cite{2025_ICLR_RandLoRA} (lemma D.2) that we rewrite here to allow this proof to be self-contained.
\begin{lemma}\label{lem:random_draws}
Let $\{X_1,\ldots,X_n\}$ denote $n$ vectors in $\R^m$ where 
$n \leq m$ drawn i.i.d from a Gaussian or uniform distribution. Then with probability $1$ $\{X_1,\ldots,X_n\}$ will be linearly independent. 
\end{lemma}
\begin{proof}
We first note that any measure defined via a Gaussian or Uniform probability distribution is absolutely continuous with respect to the Lebesgue measure. Meaning they have the same sets of measure zero as the Lebesgue measure.

We then prove the case that $\{X_1,\ldots,X_n\}$ are vectors of unit length. Since the vectors were drawn independently, we can first assume we drew $X_1$. The probability that this is the zero vector is $0$ w.r.t the Lebesgue measure on the closed unit ball $B_N(0)$ about the origin in $\R^N$ and hence any other measure absolutely continuous to it. Then draw $X_2$ and note that the probability that $X_2$ lies in 
$span\{X_1\} \cap B_N(0)$ is also $0$ since $span\{X_1\} \cap B_N(0)$ forms a set of $0$ Lebesgue measure in $B_N(0)$. Continuing in this way we find that $\{X_1,\ldots,X_n\}$ will be linearly independent with probability $1$.

For the general case where $\{X_1,\ldots,X_n\}$ are not drawn to have unit length i.e. drawn on the sphere in $\R^N$, we simply note that we can draw each one and then divide by its norm producing one of unit length. Since normalizing by the norm doesn't affect linear independence we get by the above case that  $\{X_1,\ldots,X_n\}$ must be linearly independent with probability $1$.
\end{proof}

We now prove theorem~\ref{theo:fullrank}.

\begin{proof}
Let $\mathbf{U} \in \mathbb{R}^{k \times d_{in}}$ and $\mathbf{V} \in \mathbb{R}^{k \times d_{in}}$ where $k \leq d_{in} \leq k^2$ be matrices whose entries are chosen i.i.d. from a standard Gaussian or uniform distribution. Since $k \leq d_{in}$ write 
$d_{in} = nk + p$ where $0 \leq p < k$ i.e. $p$ is the remainder when we divide $d_{in}$ by 
$k$. Note that since the entries of $U$ and $V$ are chosen i.i.d from a Gaussian or uniform distribution we have with probability $1$ that none of the columns of $U$ are multiples of each other and none of the columns of $V$ are multiples of each other. Furthermore, 
using lemma \ref{lem:random_draws} we have with probability $1$ that the $k$ column vectors $\{U_1,\ldots , U_{k}\}$ are linearly independent, as well as the second $k$  column 
vectors $\{U_{k+1},\ldots , U_{2k}\}$, and continuing in this way each batch of $k$ column vectors
$\{U_{(i-1)k+1},\ldots , U_{ik}\}$ for 
$1 \leq i \leq n$ are linearly independent and the final 
$p$ vectors 
$\{U_{nk+1},\ldots , U_{nk+p}\}$ are linearly independent. We can also assume the same for the columns vectors of $V$. 

We now observe that because 
\begin{equation}
\{U_{(i-1)k+1},\ldots , U_{ik}\}    
\end{equation}
is linearly independent  and 
\begin{equation}
 \{V_{(i-1)k+1},\ldots , V_{ik}\}   
\end{equation}
 is linearly independent  for $1 \leq i \leq n$ and 
\begin{equation}
\{U_{nk+1},\ldots , U_{nk+p}\}    
\end{equation}
are linearly independent and 
\begin{equation}
\{V_{nk+1},\ldots , V_{nk+p}\}    
\end{equation}
 are linearly independent. We have that 
 \begin{equation}
 \{U_{(i-1)k+1}\otimes V_{(i-1)k+1},\ldots , U_{ik}\otimes U_{ik}\}    
 \end{equation}
 are linearly independent for 
$1 \leq i \leq n$ and that 
\begin{equation}
\{U_{nk+1}\otimes V_{nk+1} ,\ldots , U_{nk+p}\otimes V_{nk+p}\}    
\end{equation}
are linearly independent. This uses the fact that given a collection of $p$ linearly independent vectors 
$x_1,\ldots ,x_p$ in $\R^q$ and another collection of $p$ linearly independent vectors \{$y_1,\ldots ,y_p$ in $\R^q$\} the collection of Kronecker products \{$x_1\otimes y_1,\ldots ,x_p\otimes y_p$\} in $\R^{q^2}$ are linearly independent.

Furthermore, since none of the column vectors of $U$ are multiplies of the others and none of the column vectors of $V$ are multiples of its others we have that the set of vectors
\begin{equation}\label{eqn:final_tensor}
\{U_{1}\otimes V_{1}, \ldots ,
U_{nk+1}\otimes U_{nk+1}, \ldots, 
U_{nk+p}\otimes U_{nk+p}\}
\end{equation}
are linearly independent. Then observe that the Khatri-Rao product 
$U \odot V$ has columns precisely given by \eqref{eqn:final_tensor} and thus the columns of $U \odot V$ are linearly independent. Since 
$d_{in} \leq k^2$ we have that $rank(U\odot V) = d_{in}$ as required.
\end{proof}

\section{Further empirical differences with Kronecker adapter\label{app:morediffkrona}}
Because Kronecker adapters (Krona) also uses a form of Kronecker products for PEFT, we propose here to highlight more differences between KRAdapter and Krona in terms of effective rank at initialization. Empirical evidence from vision and language tasks presented in the main paper indicates that KRAdapter consistently attains higher effective ranks than Krona given comparable trainable parameter budgets. While a comprehensive theoretical justification for this observation is beyond the scope of this empirical study, we undertake a controlled numerical experiment to analyze the effective rank and singular value distribution resulting from matrix approximation using both Kronecker and Khatri-Rao products.  Specifically, we generate random matrices of dimensions relevant to transformer architectures, ranging from ViT-L/14 attention heads (768, 1024) to LLama 3.1 attention heads (4096, 4096). We configure KRAdapter-style factorization based on its inherent shape-dependent parameter allocation.  To ensure a fair comparison, we then tune Krona's hyperparameters to precisely match the number of trainable parameters used by the KRAdapter configuration. The parameters are then initialized using a kaiminig uniform initialization strategy. Figure~\ref{fig:kravskronasvd} presents the singular value decomposition and effective rank for both factorization methods across these matrix dimensions.  Our analysis reveals that the Khatri-Rao product yields a consistently higher effective rank (10-50\%) and a more gradual decay in the singular value spectrum compared to the Kronecker product.  This suggests a more uniform distribution of singular values, indicative of richer representational capacity. These empirical findings substantiate the observed performance advantages of KRAdapter in vision and language tasks, highlighting the superior effective rank achieved by Khatri-Rao product factorization for parameter-efficient adaptation.

\begin{figure*}[h]
\includegraphics[width=\textwidth]{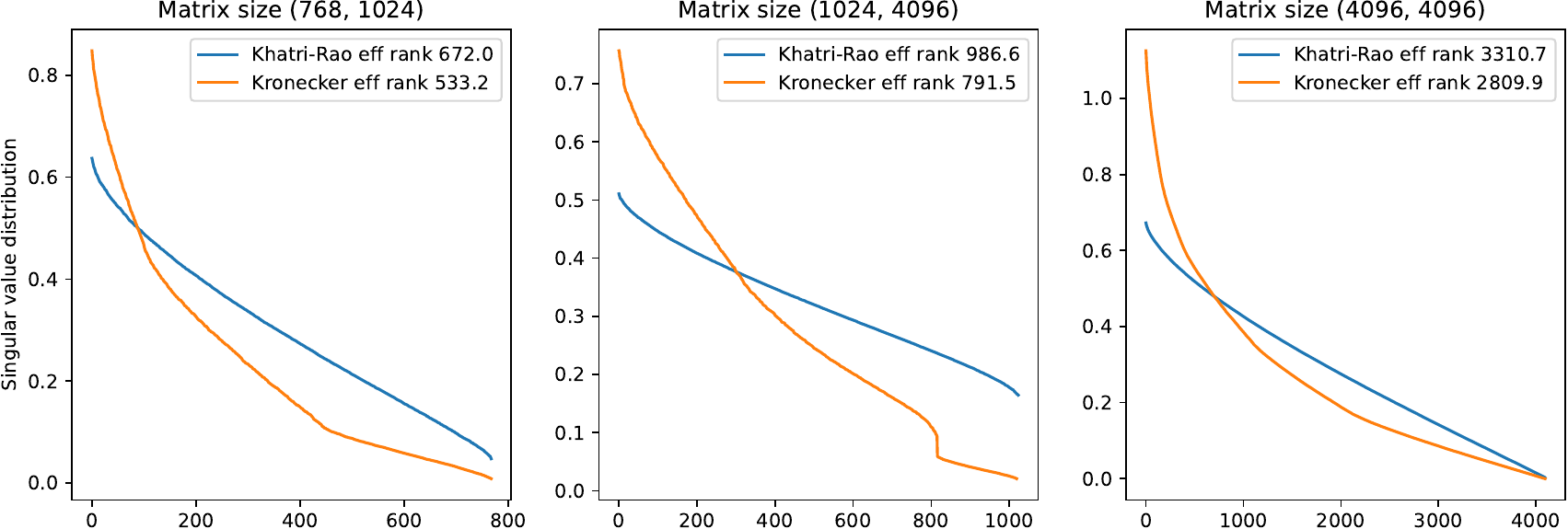}    
\caption{We compare the singular value distribution and effective rank resulting from a parameter-efficient construction of a matrix of set size using either Khatri-Rao or Kronecker products. For an equivalent amount of randomly initialized parameters, the Khatri-Rao produces a matrix with a smooth, more balanced svd sprectrum, resulting in a higher effective rank.
\label{fig:kravskronasvd}}
\end{figure*}

\section{Details about the toy experiments~\label{app:toy}}
\subsection{Training details}
The matrices used in our experiments have a size of $1024\times 768$. We aim to match the number of parameters of our proposed KRAdapter as closely as possible. Minor discrepancies in parameter counts for other methods arise due to the inherent structural differences of each adaptation technique. Specifically, the number of parameters for each method is: LoRA and SinLoRA: $50,176$, Krona: $50,700$, RandLoRA: $49,168$, and KRAdapter: $49,152$. We train models using the AdamW optimizer~\footnote{https://pytorch.org/docs/stable/generated/torch.optim.AdamW.html} for $100$ iterations with a fixed learning rate of $10^{-2}$. The AdamW optimizer is used with default parameters ($\beta_1=0.9$, $\beta_2=0.999$, weight decay=0.01). Our training objective is to minimize the mean of the squared error between the predicted and target matrices.

\subsection{Matrix generation}
We generate six different synthetic patterns, each designed to probe specific aspects of parameter-efficient fine-tuning algorithms. \textbf{Normally-distributed Random Matrix} generated from a standard normal distribution. This serves as a baseline representing a high-rank weight matrix, testing the algorithms' general approximation capability. \textbf{Sparse Random Matrix (90\% Sparsity)}, a normally distributed random matrix where 90\% of elements are randomly set to zero. This baseline simulates scenarios where pre-trained models contains crucial parameters that should not be modified during fine-tuning. \textbf{PCA-Whitened Random Matrix}, a random matrix transformed using Principal Component Analysis (PCA) whitening. This process de-correlates the random features, assessing how well algorithms can generate highly de-correlated representations. \textbf{Low-Rank Matrix} constructed by taking a normally distributed random matrix and zeroing out all but the top one fourth of singular values. Tests the ability of full-rank algorithms to model low-rank matrices. \textbf{CLIP ImagNet fine-tune delta (Vision or Language)}, obtained by the element-wise difference between the pre-trained CLIP-ViT-L/14 weights and the weights obtained after standard fine-tuning on ImageNet (also known as task vector~\cite{2024_NeurIPS_atlas}). The weight difference is extracted from the last attention layer of either the vision or language backbone. This pattern represents a realistic target weight for LoRA when adapting pre-trained transformer weights, allowing to assess performance on real-world fine-tuning. \textbf{High/Low frequency features} where rows are generated using up to $5$ superposed sinusoidal functions, with frequencies linearly increasing along the rows. The high frequencies are contained between $[1000, 10000]$ Hz while the low frequencies are contained between $[1, 100]$ Hz. This structured pattern assesses the algorithms' bias towards feature frequencies.

For the normally-distributed random and identity matrices, we respectively use the \texttt{torch.randn}~\footnote{https://pytorch.org/docs/main/generated/torch.randn.html} and \texttt{torch.eye}~\footnote{https://pytorch.org/docs/main/generated/torch.eye.html} functions to generate matrices of the desired size.

\textbf{PCA-Whitened Random Matrix:} We generate a normally-distributed random matrix using \texttt{torch.randn} and then perform PCA whitening. This involves multiplying the data by the eigenvectors of the covariance matrix, effectively decorrelating the features. We then scale each row of the resulting matrix by the square root of the corresponding eigenvalue to normalize the variance.

\textbf{High/Low frequency features:} Each row of the matrices is generated by sampling a sinusoid function $\mathbf{f}(f, t) = \sin{(2\pi f t)}$ over one second. The frequency $f$ increases linearly from $1$ Hz for the first row to $1,000$ Hz for the last row ($1,000$ to $10,000$ for the high frequencies). This creates a matrix where each row represents a sinusoid with a different frequency.

\subsection{Visualization}
We propose a visualization of the achieved reconstruction for each PEFT algorithm in Figure~\ref{fig:vis-toy} for smaller $128\times 128$ matrices. For the fine-tuned weights, we select the first $128$ rows and columns.
\begin{figure*}[h]
\centering
\includegraphics[width=.8\textwidth]{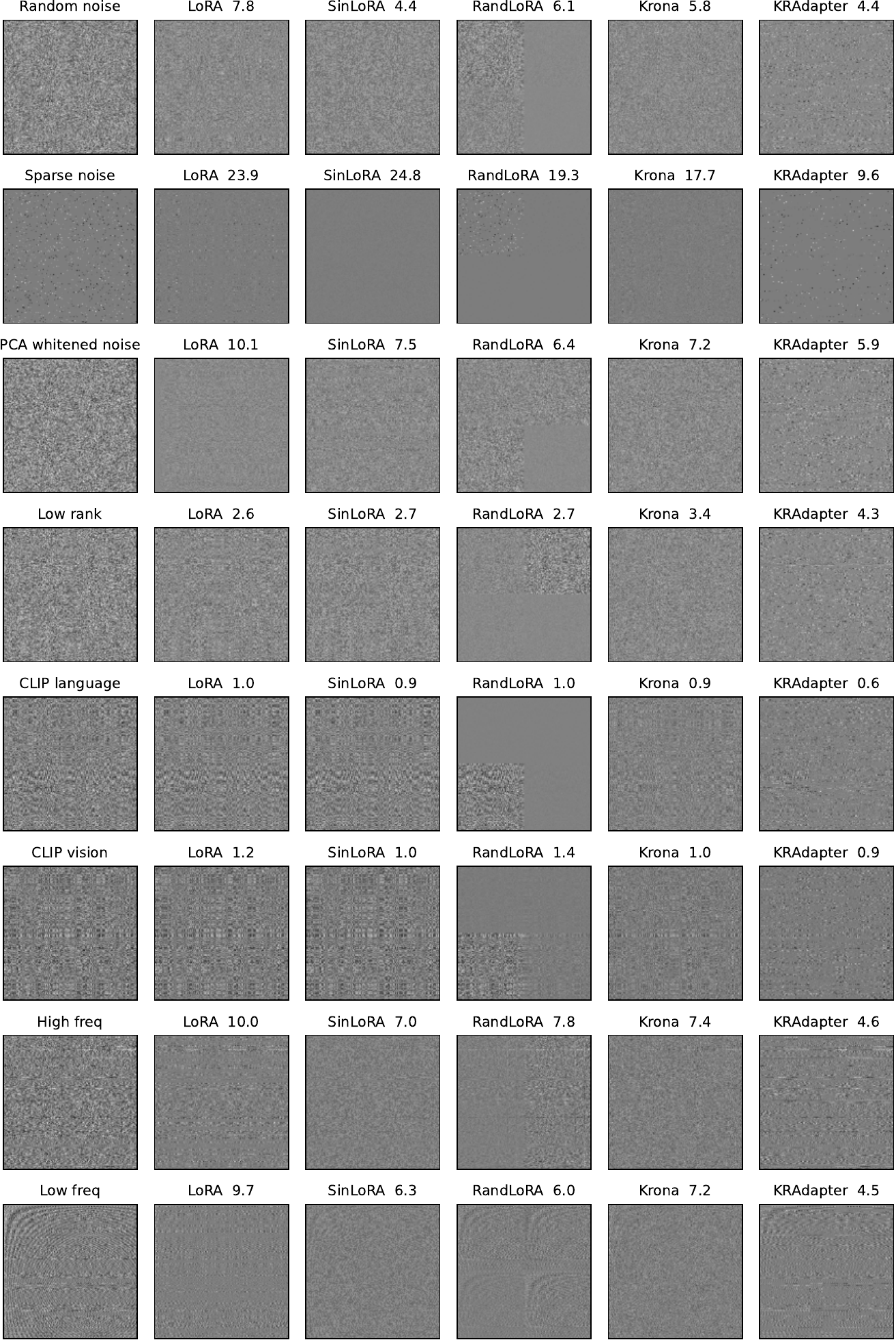}  
\caption{Toy experiment. We evaluate the capacity of PEFT methods to produce specific types of weight matrices. We report the generated matrices according to the target (left) and the absolute element-wise nuclear error. Lower is better. All algorithms train at least the same amount of parameters as KRAdapter}
\label{fig:vis-toy}
\end{figure*}

\section{Effective rank~\label{app:effrank}} To further investigate the intrinsic dimensionality of each method we report the average effective ranks averaged across attention layers post fine-tuning in Table~\ref{tab:visionrank}. Specifically, the effective rank~\cite{2007_ESPC_effectiverank} of a matrix $M$ is calculated as $\exp({-\sum_i{S_i^n\log{S_i^n}}})$, where $S_i^n=\frac{S_i}{\sum_i{S_i}}$ represents the sum-normalized singular values of $M$. An effective rank close to the mathematical rank indicates that the weight matrix makes full use of the available spectrum to significantly modify the space in a wide range of directions. We report that the effective rank of KRAdapter is consistently higher than that of other theoretically full-rank algorithms.

\begin{table}[h!]
    \centering
    \setlength{\tabcolsep}{0.5pt}
    \begin{tabular}{lcccccc}
        \toprule
        & \rotatebox[origin=c]{45}{LoRA} & \rotatebox[origin=c]{45}{SinLoRA} & \rotatebox[origin=c]{45}{RandLoRA} & \rotatebox[origin=c]{45}{Krona} & \rotatebox[origin=c]{45}{KRAdapter} & \rotatebox[origin=c]{45}{Max rank} \\
        \midrule
        ViT-B-32 & 4.5 & 21.9 & 494.8 & 518.5 & \textbf{705.9} & 768 \\
        ViT-L-16 & 13.1 & 31.8 & 587.0 & 744.0 & \textbf{959.7} & 1024 \\
        LLama3.1 & 16.8 & 24.0 & 562.3 & 734.2 & \textbf{970.6} & 1024\\
        Qwen2.5-7B & 8.5 & 18.7 & 247.6 & 310.5 & \textbf{486.6}& 512 \\
        \bottomrule
    \end{tabular}
    \caption{Effective ranks of full-rank PEFT algorithms for vision or language architectures.}
    \label{tab:visionrank}
\end{table}

\section{Training times and VRAM usage}
We find that all algorithms use comparable amounts of VRAM during training except for RandLoRA which comes at the cost of a slight increase in training time. We report training time results in Table~\ref{tab:training_time_comparison} for various ViT architecture for 1 epoch on ImageNet and LLama3-8b for the commonsense reasoning task for 4 epochs (160k samples in total). Note that PEFT algorithms are trained on attention layers only. Although not reported in this table, DoRA's training time is comparable to RandLoRA.

\begin{table}[htbp]
    \centering
    
    \begin{tabular}{lcccccc}
        \toprule
        Algorithm & ViT-B/32 & ViT-L/14 & ViT-H/14 & LLama3-8B & Qwen2.5-7B \\
        \midrule
        LoRA      & 16.8 mins & 134.1 mins  & 215.5 mins  & 243.3 mins & 222.2 mins \\
        SinLoRA   & 16.8 mins & 136.9 mins & 216.6 mins  & 246.2 mins & 224.3 mins \\
        RandLoRA  & 16.7 mins & 138.4 mins  & 225.5 mins  & 265.3 mins & 235.2 mins \\
        Krona     & 16.6 mins & 135.9 mins  & 217.2 mins  & 250.4 mins & 227.5 mins \\
        KRAdapter & 16.5 mins & 137.5 mins  & 220.1 mins  & 247.6 mins & 226.3 mins \\
        FT        & 21.1 mins & 167.9 mins & 270.5 mins & Not trained & Not trained \\
        \bottomrule
    \end{tabular}
    \caption{Comparison of training time for one epoch on ImageNet on CLIP architectures and LLama3-8B, Qwen2.5-7B for one epoch on the commonsense reasoning dataset}
    \label{tab:training_time_comparison}
    \vspace{0.1cm}
    
\end{table}

\section{CLIP classification\label{app:classifapp}}

\subsection{Dataset details}
We fine-tune pre-trained vision-language architectures on $11$ vision datasets. For few-shot learning experiments, we consistently train models for 10 epochs. In contrast, for 50\% and 100\% fine-tuning scenarios, we follow~\cite{2024_NeurIPS_atlas,2025_ICLR_RandLoRA} and adjust the number of training epochs for the full fine-tuning baseline based on convergence behavior, aiming for optimal performance. We do not perform early stopping as we do not observe significant over-fitting. All algorithms use the same training samples and training epochs. Detailed specifications of the $11$ datasets, including number of training samples and the specific number of epochs used, are reported in Table~\ref{tab:datasets}.

\begin{table}[h!]    
    \centering
    \setlength{\tabcolsep}{2pt} 
    \scriptsize
    \begin{tabular}{clccccccccccccccccc}
        \toprule
        \# & Datasets & Classes & \multicolumn{3}{c}{Splits} & Epochs \\
        \cline{4-6} \\ [-5pt]
        & & & \textit{train} & \textit{val} & \textit{test} &\\
        \midrule
        (1) & Cars & 196 & 7,330 & 814 & 8,041 & 35 \\
        (2) & DTD & 47 & 3,384 & 376 & 1,880 & 76 \\
        (3) & EuroSAT & 10 & 21,600 & 2,700 & 2,700 & 12 \\
        (4) & SUN397 & 397 & 17,865 & 1,985 & 19,850 & 14 \\
        (5) & Food101 & 101 & 70,750 & 5,000 & 25,250 & 15 \\
        (6) & Caltech101  & 101 & 6,941 & 694 & 1,736 & 10 \\
        (7) & FGVCAircraft & 100 & 3,334 & 3,333 & 3,333 & 60 \\
        (8) & Flowers102 & 102 & 1,020 & 1,020 & 6,149 & 40 \\
        (9) & OxfordIIITPet & 37 & 3,312 & 368 & 3,669 & 5 \\
        (10) & UCF101  & 101 & 7,639 & 1,898 & 3,783 & 20 \\
        (11) & ImageNet & 1,000 & 1,276,167 & 5,000 & 50,000 & 1 \\
        \bottomrule
    \end{tabular}
    \captionsetup{font=footnotesize}
    \caption{Vision datasets used for the image classification experiments}
    \label{tab:datasets}    
\end{table}
\subsection{Classic datasets}
We report detailed average results for the classic datasets of Table~\ref{tab:datasets} for ViT-B/32, ViT-L/14 and ViT-H/14 in Table~\ref{tab:image_class}.
\begin{table*}[htbp]
    \centering
    \setlength{\tabcolsep}{1.5pt}    
    \begin{tabular}{lcccccccccccccccccc}
        \toprule
        & \multicolumn{6}{c}{ViT-B/32} & \multicolumn{6}{c}{ViT-L/14} & \multicolumn{6}{c}{ViT-H/14} \\
        \cmidrule(lr){2-7} \cmidrule(lr){8-13} \cmidrule(lr){14-19} 
        Shots & 1 & 4 & 16 & 50\% & 100\% & Avg. & 1 & 4 & 16 & 50\% & 100\% & Avg.  & 1 & 4 & 16 & 50\% & 100\% & Avg.\\
        \midrule
        LoRA & 60.93 & 66.11 & 69.47 & 74.53 & 77.48 & 69.70 & 74.82 & 78.65 & 81.64 & 85.59 & 88.17 & 81.77 & 79.82 & 80.91 & 83.00 & 86.39 & 88.51 & 83.73 \\
        SinLoRA & 60.36 & 67.93 & 72.31 & 75.90 & 78.38 & 70.98 & 75.43 & 80.09 & 82.69 & 86.14 & 88.03 & 82.48 & 79.96 & 82.59 & 84.66 & 86.49 & 88.22 & 84.38 \\
        RandLoRA & 59.40 & 68.98 & 73.91 & 78.57 & 81.99 & 72.57 & 76.26 & 81.60 & 84.28 & 87.92 & 89.93 & 84.00 & 81.40 & 84.19 & 86.52 & 89.48 & 90.83 & 86.48 \\
        Krona & 58.64 & 68.94 & 73.86 & 78.05 & 81.12 & 72.12 & 75.75 & 81.52 & 84.47 & 88.11 & 89.85 & 83.94 & 79.74 & 84.03 & 86.68 & 89.62 & 90.81 & 86.18 \\          
        KRAdapter & 58.86 & 69.28 & 74.80 & 79.67 & 82.74 & 73.07 & 76.39 & 81.97 & 85.14 & 88.79 & 90.46 & 84.55 & 81.18 & 84.75 & 87.10 & 89.62 & 90.76 & 86.68 \\        
        \midrule
        FT & 58.90 & 70.03 & 75.52 & 80.31 & 83.42 & 73.64 & 77.39 & 80.96 & 84.97 & 87.91 & 90.03 & 84.25 & 77.39 & 80.96 & 84.97 & 87.91 & 90.03 & 83.65 \\    
        \bottomrule
    \end{tabular}    
    \caption{Parameter-efficient vision-language CLIP tuning for image classification.\label{tab:image_class}}
\end{table*}
\subsection{VTAB1k\label{app:vtab1kdatasets}}

The Visual Task Adaptation Benchmark (VTAB)~\cite{2019_arxiv_vtab1k} is a collection of datasets used to evaluate the capacity of PEFT algorithms to adapt large pretrained models to 3 categories of tasks.
\subsubsection{Dataset presentation}
\paragraph{Natural subset}
\textit{Caltech101} \cite{2022_misc_caltech101} focuses on classifying images of 102 object categories, including common objects and a background class.
\textit{CIFAR-100} \cite{2009_CIFAR} is a natural image classification dataset with 100 classes.
The \textit{DTD} dataset \cite{2014_CVPR_DTD} involves classifying textural patterns across 47 classes.
\textit{Flowers102} \cite{2020_arxiv_Flowers102} is dedicated to classifying 102 flower species found in the UK.
\textit{Pets} \cite{2012_CVPR_pets} is a dataset for classifying cat and dog breeds, containing 37 classes.
\textit{Sun397} \cite{2010_CVPR_SUN397} is a scenery classification benchmark with 397 hierarchically structured classes.

\paragraph{Specialized subset}
\textit{SVHN} \cite{2011_NeurIPSW_svhn} is a dataset for classifying street-view house numbers with 10 classes.
\textit{EuroSAT} \cite{2018_IGRSS_eurosat} consists of Sentinel-2 satellite imagery for land use classification into 10 classes.
\textit{Resisc45} \cite{2017_IEEE_resisc45} is a remote sensing image scene classification dataset with 45 classes.
\textit{Patch Camelyon} \cite{2018_arxiv_patchcamelyon} is a large dataset of histopathologic scans for binary classification of metastatic tissue presence.
The \textit{Retinopathy} dataset \cite{2015_kaagle_retinopathy} focuses on predicting the severity of Diabetic Retinopathy on a 0-4 scale from high-resolution retina images.

\paragraph{Structured subset} The \textit{CLEVR}~\cite{2017_CVPR_CLEVR} datasets utilize images from a visual question answering task, with the 'count' variant predicting the number of objects and the 'distance' variant predicting the depth of the closest object. The \textit{dSprites}~\cite{2017_git_DSPRITES} dataset, originally designed for disentanglement learning, is repurposed for location and orientation prediction tasks of simple 2D shapes. Similarly, the \textit{SmallNORB}~\cite{2004_CVPR_SmallNORB} dataset, containing images of 3D toys, is used for predicting azimuth and elevation angles of the objects. \textit{DMLab}~\cite{2016_arxiv_DMLab} provides 3D navigation environments where the task is to classify distances to reward objects, and finally, \textit{KITTI}~\cite{2013_IJRR_Kitti} involves predicting the depth of vehicles in real-world driving scenes. These tasks require models to reason about object counts, distances, orientations, and locations, spanning both 2D and 3D visual understanding which presents significant challenges for CLIP architectures.
\subsubsection{Prompt design}
Although the prompts design for the natural and part fo the specialized subset is straight forward, these are not evident for the structed subset especially when the classification is discrete. We settle on the class names described in Table~\ref{tab:clip_prompts} as we find they perform better than random for the zero-shot models and allow to see an improved performance with stronger zero-shot models. The final prompt we train CLIP with is "An image of a \{classname\}.' and we train with the SimCLR~\cite{2020_ICML_SimCLR} augmentations.

\begin{sidewaystable}[htbp]
    \centering
   
    \begin{tabularx}{\textwidth}{lX}
        \toprule
        Dataset & Class names \\
        \midrule
        CAMELYON & \texttt{'with no metastatic tissue', 'containing metastatic tissue'} \\
        RETINOPATHY & \texttt{'with no diabetic retinopathy', 'with mild diabetic retinopathy', 'with moderate diabetic retinopathy', 'with severe diabetic retinopathy', 'with extreme diabetic retinopathy'} \\
        CLEVR\_COUNT & \texttt{'3 items', '4 items', '5 items', '6 items', '7 items', '8 items', '9 items', '10 items'} \\
        CLEVR\_DIST & \texttt{'congested', 'larger', 'large', 'normal', 'small', 'tiny'} \\
        DSPRITES\_LOC & \texttt{'0-6 percent x axis', '6-12 percent x axis', '12-18 percent x axis', '18-25 percent x axis', '25-31 percent x axis', '31-37 percent x axis', '37-43 percent x axis',  '43-50 percent x axis', '50-56 percent x axis', '56-62 percent x axis', '62-68 percent x axis', '68-75 percent x axis', '75-81 percent x axis', '81-87 percent x axis', '87-93 percent x axis','93-100 percent x axis'} \\
        DSPRITES\_ORIENT & \texttt{'shape rotated 0-22.5 degrees clockwise', 'shape rotated 22.5-45.0 degrees clockwise', 'shape rotated 45.0-67.5 degrees clockwise', 'shape rotated 67.5-90.0 degrees clockwise', 'shape rotated 90.0-112.5 degrees clockwise', 'shape rotated 112.5-135.0 degrees clockwise', 'shape rotated 135.0-157.5 degrees clockwise', 'shape rotated 157.5-180.0 degrees clockwise', 'shape rotated 180.0-202.5 degrees clockwise', 'shape rotated 202.5-225.0 degrees clockwise', 'shape rotated 225.0-247.5 degrees clockwise', 'shape rotated 247.5-270.0 degrees clockwise', 'shape rotated 270.0-292.5 degrees clockwise', 'shape rotated 292.5-315.0 degrees clockwise', 'shape rotated 315.0-337.5 degrees clockwise', 'shape rotated 337.5-360.0 degrees clockwise'} \\
        SMALLNORB\_AZIMUT & \texttt{'shape rotated by 0-20 degrees clockwise',  'shape rotated by 20-40 degrees clockwise',  'shape rotated by 40-60 degrees clockwise',  'shape rotated by 60-80 degrees clockwise',  'shape rotated by 80-100 degrees clockwise',  'shape rotated by 100-120 degrees clockwise',  'shape rotated by 120-140 degrees clockwise',  'shape rotated by 140-160 degrees clockwise',  'shape rotated by 160-180 degrees clockwise',  'shape rotated by 180-200 degrees clockwise',  'shape rotated by 200-220 degrees clockwise',  'shape rotated by 220-240 degrees clockwise',  'shape rotated by 240-260 degrees clockwise',  'shape rotated by 260-280 degrees clockwise',  'shape rotated by 280-300 degrees clockwise',  'shape rotated by 300-320 degrees clockwise',  'shape rotated by 320-340 degrees clockwise',  'shape rotated by 340-360 degrees clockwise'} \\
        SMALLNORB\_ELEVATION & \texttt{'object photographed with a 30 degrees elevation', 'object photographed with a 35 degrees elevation', 'object photographed with a 40 degrees elevation', 'object photographed with a 45 degrees elevation', 'object photographed with a 50 degrees elevation', 'object photographed with a 55 degrees elevation', 'object photographed with a 60 degrees elevation', 'object photographed with a 65 degrees elevation', 'object photographed with a 70 degrees elevation'} \\
        DMLAB & \texttt{'obstructed', 'large', 'bigger', 'normal', 'smallest', 'empty'} \\
        KITTI & \texttt{'congested', 'close', 'distant', 'empty'} \\
        \bottomrule
    \end{tabularx}
     \caption{CLIP Prompts for the Structed subset of VTAB-1k and + Camelyon and Retinopathy. We train the VL CLIP models with the prompt "An image of a \{classname\}."}
    \label{tab:clip_prompts}
\end{sidewaystable}

\subsubsection{Detailed results\label{app:vtabdetail}}
Tables \ref{tab:vtab1kvitb32} report per dataset detailed results for PEFT algorithms using the ViT-\{B/32,L/14,H/14\} architectures respectively
\begin{table*}[htbp]
\centering
\setlength{\tabcolsep}{1pt}
\begin{tabular}{lccccccccccccccccccccccccc}
\toprule
 & \multicolumn{7}{c}{\textbf{Natural}} & \multicolumn{4}{c}{\textbf{Specialized}} & \multicolumn{8}{c}{\textbf{Structured}} \\
\cmidrule(lr){2-8} \cmidrule(lr){9-12} \cmidrule(lr){13-20}
 & \rotatebox{90}{CIFAR-100} & \rotatebox{90}{Caltech101} & \rotatebox{90}{DTD} & \rotatebox{90}{Flowers102} & \rotatebox{90}{Pets} & \rotatebox{90}{SVNH} & \rotatebox{90}{Sun397} & \rotatebox{90}{Camelyon} & \rotatebox{90}{EuroSAT} & \rotatebox{90}{Resisc45} & \rotatebox{90}{Retinopathy} & \rotatebox{90}{Clevr-Count} & \rotatebox{90}{Clevr-Dist} & \rotatebox{90}{dSpr-Loc} & \rotatebox{90}{dSpr-Ori} & \rotatebox{90}{sNORB-Azim} & \rotatebox{90}{sNORB-Ele} & \rotatebox{90}{DMLab} & \rotatebox{90}{KITTI-Dist} & \rotatebox{90}{\textbf{Mean Nat.}} & \rotatebox{90}{\textbf{Mean Spe}} & \rotatebox{90}{\textbf{Mean Struc}} & \rotatebox{90}{\textbf{Group Mean}} & \rotatebox{90}{\textbf{All Mean}} \\
\midrule
\multicolumn{25}{l}{\textbf{ViT-B/32}} \\
\midrule
Zero-shot & 41.5 & 78.8 & 41.7 & 66.7 & 87.7 & 25.8 & 59.5 & 60.7 & 31.4 & 53.8 & 55.7 & 25.1 & 17.4 & 6.7 & 8.1 & 6.2 & 11.6 & 20.2 & 39.8 & 57.4 & 50.4 & 16.9 & 41.6 & 38.9 \\
LoRA & 48.6 & 84.0 & 60.4 & 76.8 & 84.0 & 89.1 & 51.3 & 83.8 & 95.0 & 83.1 & 68.5 & 67.4 & 45.1 & 36.8 & 46.1 & 22.4 & 39.9 & 50.5 & 53.3 & 70.6 & 82.6 & 45.2 & 66.1 & 62.4 \\
SinLoRA & 49.0 & 84.7 & 60.6 & 84.2 & 84.4 & 90.3 & 50.6 & 84.5 & 95.0 & 83.9 & 69.0 & 69.3 & 53.1 & 78.5 & 51.5 & 22.4 & 41.2 & 51.9 & 56.8 & 72.0 & 83.1 & 53.1 & 69.4 & 66.4 \\
RandLoRA & 48.6 & 87.4 & 68.7 & 88.0 & 85.9 & 91.7 & 49.0 & 84.8 & 93.0 & 87.3 & 64.6 & 63.8 & 58.1 & 82.0 & 54.3 & 23.1 & 32.8 & 54.3 & 56.5 & 74.2 & 82.4 & 53.1 & 69.9 & 67.1 \\
Krona & 48.5 & 86.7 & 66.5 & 86.3 & 86.0 & 91.6 & 50.1 & 84.5 & 93.4 & 86.5 & 69.3 & 70.5 & 57.4 & 82.2 & 53.7 & 23.4 & 29.4 & 54.8 & 54.6 & 73.7 & 83.4 & 53.3 & 70.1 & 67.1 \\
KRAdapter & 52.3 & 88.4 & 70.3 & 88.9 & 87.0 & 91.7 & 53.5 & 84.9 & 93.1 & 88.4 & 69.4 & 69.4 & 58.3 & 79.9 & 54.0 & 25.3 & 32.1 & 54.1 & 53.0 & 76.0 & 84.0 & 53.3 & 71.1 & 68.1 \\
\cmidrule(r){2-25}
FT & 51.2 & 87.2 & 68.1 & 89.0 & 85.7 & 92.0 & 56.4 & 84.6 & 93.1 & 87.2 & 69.2 & 68.5 & 58.2 & 80.3 & 54.7 & 24.2 & 32.1 & 54.1 & 53.0 & 75.7 & 83.5 & 53.1 & 70.8 & 67.8 \\
\midrule
\multicolumn{25}{l}{\textbf{ViT-L/14}} \\
\midrule
Zero-shot & 55.9 & 80.9 & 52.5 & 78.9 & 93.3 & 56.4 & 64.2 & 54.7 & 42.6 & 66.2 & 23.9 & 19.0 & 22.5 & 6.6 & 6.8 & 5.5 & 9.3 & 21.1 & 17.6 & 68.9 & 46.9 & 13.6 & 43.1 & 40.9 \\
LoRA & 64.9 & 87.9 & 75.2 & 96.8 & 92.6 & 94.7 & 63.9 & 86.0 & 95.4 & 91.8 & 74.7 & 85.5 & 43.4 & 74.0 & 54.5 & 22.0 & 39.8 & 58.1 & 60.1 & 82.3 & 87.0 & 54.7 & 74.6 & 71.6 \\
SinLoRA & 65.9 & 88.2 & 75.7 & 97.0 & 92.5 & 94.7 & 63.7 & 86.8 & 95.6 & 92.0 & 72.8 & 86.5 & 45.5 & 84.7 & 60.9 & 22.1 & 38.0 & 60.4 & 56.4 & 82.5 & 86.8 & 56.8 & 75.4 & 72.6 \\
RandLoRA & 63.4 & 89.1 & 76.7 & 97.3 & 93.8 & 94.7 & 64.5 & 86.7 & 94.8 & 92.2 & 74.2 & 81.0 & 60.2 & 84.5 & 60.9 & 25.8 & 35.7 & 60.4 & 58.1 & 82.8 & 87.0 & 58.3 & 76.0 & 73.4 \\
Krona & 64.5 & 89.8 & 77.1 & 97.5 & 92.9 & 95.4 & 66.9 & 86.9 & 95.4 & 92.7 & 74.5 & 79.1 & 56.4 & 83.9 & 61.5 & 26.5 & 35.7 & 58.9 & 58.1 & 83.4 & 87.4 & 57.5 & 76.1 & 73.3 \\
KRAdapter & 70.0 & 89.8 & 78.5 & 98.0 & 93.7 & 95.2 & 68.5 & 86.5 & 95.1 & 93.0 & 73.0 & 81.7 & 55.3 & 79.1 & 62.0 & 25.3 & 34.5 & 59.5 & 60.1 & 84.8 & 86.9 & 57.2 & 76.3 & 73.6 \\\cmidrule(r){2-25}
FT & 63.2 & 89.9 & 76.0 & 97.8 & 92.8 & 94.2 & 68.1 & 86.6 & 95.4 & 92.3 & 75.1 & 81.8 & 49.9 & 84.7 & 64.6 & 27.2 & 37.2 & 60.4 & 60.3 & 83.2 & 87.4 & 58.3 & 76.3 & 73.6 \\

\midrule
\multicolumn{25}{l}{\textbf{ViT-H/14}} \\
\midrule
Zero-shot & 65.9 & 83.4 & 63.5 & 79.6 & 94.6 & 45.5 & 74.7 & 54.5 & 52.9 & 70.9 & 23.4 & 34.9 & 22.6 & 6.1 & 8.9 & 5.9 & 11.2 & 15.2 & 37.8 & 72.4 & 50.4 & 17.8 & 46.9 & 44.8\\
LoRA & 68.9 & 89.5 & 78.6 & 97.4 & 92.2 & 94.6 & 67.6 & 86.9 & 95.7 & 91.9 & 72.4 & 79.7 & 40.8 & 86.6 & 62.6 & 26.6 & 39.2 & 58.3 & 60.3 & 84.1 & 86.7 & 56.8 & 75.9 & 73.1 \\
SinLoRA & 69.0 & 89.8 & 78.0 & 97.4 & 92.3 & 94.7 & 68.5 & 87.3 & 95.5 & 92.0 & 72.7 & 87.1 & 41.8 & 87.4 & 62.9 & 28.3 & 39.2 & 60.0 & 58.6 & 84.2 & 86.9 & 58.2 & 76.4 & 73.8 \\
RandLoRA & 66.4 & 90.8 & 79.0 & 97.2 & 92.2 & 94.0 & 67.6 & 86.8 & 95.6 & 92.0 & 74.5 & 84.0 & 59.1 & 85.4 & 60.5 & 28.9 & 38.4 & 60.5 & 58.8 & 83.9 & 87.3 & 59.5 & 76.9 & 74.3 \\
Krona & 68.1 & 92.2 & 79.6 & 97.7 & 92.6 & 94.7 & 69.7 & 87.5 & 95.0 & 92.5 & 72.0 & 80.6 & 54.4 & 86.3 & 61.2 & 28.4 & 36.8 & 59.0 & 57.7 & 84.9 & 86.7 & 58.0 & 76.6 & 74.0 \\
KRAdapter & 71.2 & 92.5 & 80.0 & 98.1 & 93.0 & 94.4 & 71.6 & 86.1 & 95.6 & 93.1 & 73.3 & 84.5 & 53.3 & 84.0 & 60.3 & 27.2 & 36.4 & 59.5 & 59.6 & 85.8 & 87.0 & 58.1 & 77.0 & 74.4 \\
\cmidrule(r){2-25}
FT & 66.5 & 90.4 & 77.8 & 97.5 & 92.5 & 94.3 & 78.9 & 84.5 & 95.4 & 88.9 & 70.3 & 76.0 & 35.0 & 44.3 & 49.4 & 14.8 & 26.2 & 47.2 & 56.4 & 85.4 & 84.8 & 43.7 & 71.3 & 67.7 \\
\bottomrule
\end{tabular}
\caption{Accuracies training on VTAB1k benchmark. We report per dataset accuracies as well as category-wise averages. Base networks are ViT CLIP models in version - B/32, L/14 and H/14 where both vision and language backbones are trained.}
\label{tab:vtab1kvitb32}
\end{table*}

\subsection{OOD datasets\label{app:ooddatasets}}
We evaluate the out-of-distribution (OOD) generalization of image classification models trained on ImageNet~\cite{2012_NeurIPS_ImageNet}, using datasets that probe model robustness under various distribution shifts with the standard ImageNet test set:
\paragraph{ImageNet-A} (Naturally Adversarial)~\cite{2021_CVPR_ImageNetA} comprises 7,500 real-world images from 200 ImageNet classes that are confidently misclassified by standard models, yet easily recognizable by humans. ImageNet-A assesses robustness to naturally occurring, subtle adversarial examples present in real-world data, highlighting vulnerabilities beyond synthetic adversarial attacks.
\paragraph{ImageNet-R} (Renditions)~\cite{2021_ICCV_ImageNetR} contains 30,000 images across 200 ImageNet classes, featuring artistic renditions like paintings, sketches, and sculptures. It evaluates robustness to significant stylistic domain shifts, testing if models generalize beyond photographic images and capture semantic content despite variations in visual style.
\paragraph{ImageNet-Sketch}~\cite{2019_NeurIPS_ImageNetSketch} presents a more extreme domain shift with 50,000 black and white sketches across all 1,000 ImageNet classes. ImageNet-Sketch serves as a stress test, evaluating a model's ability to generalize to drastically different image modalities and rely on high-level semantic understanding rather than low-level image features.
\paragraph{ImageNet-v2}~\cite{2019_ICML_ImageNetV2} is not an OOD dataset in the same sense but an updated test set collected using the original ImageNet methodology. It aims to provide a more reliable evaluation by mitigating potential test set contamination and overfitting to the original ImageNet validation set. We study three subsets including "Freq" (Matched Frequency) which replicates the original validation set's label distribution, "Top" (Top-5 Accuracy Matched) which matches the top-5 accuracy of a reference model, and "Thresh" (Thresholded) which uses a higher worker agreement threshold for potentially cleaner labels.

\subsubsection{Detailed OOD results}
Table~\ref{tab:ooddetail} reports detailed per-dataset accuracies for the OOD experiments on ImageNet. 
\begin{table}[h]
\centering
\setlength{\tabcolsep}{1.5pt}
\begin{tabular}{lccccccccccccccc}
\toprule
 & \rotatebox{90}{ImageNet \textbf{(ID)}} & \rotatebox{90}{ImageNetA} & \rotatebox{90}{ImageNetSketch} & \rotatebox{90}{ImageNetR} & \rotatebox{90}{ImageNetV2Thresh} & \rotatebox{90}{ImageNetV2Top} & \rotatebox{90}{ImageNetV2Freq} & \rotatebox{90}{CIFAR100} & \rotatebox{90}{\textbf{OOD}} & \rotatebox{90}{\textbf{Improve ID}} & \rotatebox{90}{\textbf{Improve OOD}} & \rotatebox{90}{\textbf{Ratio}} & \rotatebox{90}{\textbf{Effrank}} & \rotatebox{90}{\textbf{Spectral}} & \rotatebox{90}{\textbf{Fro}} \\
\midrule
\multicolumn{16}{l}{\textbf{ViT-B/32}} \\
\midrule
ZS & 62.64 & 32.28 & 40.78 & 66.56 & 62.93 & 68.23 & 55.28 & 62.26 & 55.47 & n/a & n/a & n/a & n/a & n/a & n/a \\
LoRA & 72.16 & 28.6 & 42.82 & 66.10 & 70.80 & 76.32 & 62.52 & 63.79 & 58.71 & 9.52 & 3.2 & 0.34 & 20.4 & 19.6 & 4.1 \\
SinLoRA & 72.84 & 28.43 & 42.39 & 64.48 & 71.34 & 76.97 & 62.44 & 64.24 & 58.61 & 10.20 & 3.1 & 0.31 & 276.8 & 223.5 & 14.7 \\
RandLoRA & 72.01 & 27.55 & 42.05 & 64.31 & 71.00 & 76.72 & 61.8 & 65.64 & 58.44 & 9.37 & 3.0 & 0.31 & 464.6 & 46.2 & 5.7 \\
Krona & 71.88 & 28.51 & 42.38 & 66.06 & 71.04 & 76.48 & 62.16 & 65.95 & 58.94 & 9.24 & 3.5 & 0.38 & 584.0 & 44.1 & 4.9 \\
KRAdapter & 72.52 & 30.32 & 43.67 & 67.84 & 71.39 & 77.23 & 62.59 & 66.32 & 59.91 & 9.88 & 4.4 & 0.45 & 696.0 & 7.3 & 2.3 \\
\midrule
FT & 75.54 & 25.71 & 42.35 & 64.31 & 73.41 & 78.7 & 64.85 & 62.68 & 58.86 & 12.9 & 3.4 & 0.26 & 590.4 & 0.7 & 0.8 \\
\midrule
\multicolumn{16}{l}{\textbf{ViT-L/14}} \\
\midrule
ZS & 75.44 & 70.77 & 59.6 & 87.73 & 75.86 & 79.05 & 69.75 & 76.15 & 74.13 & n/a & n/a & n/a & n/a & n/a & n/a \\
LoRA & 83.34 & 70.73 & 61.36 & 86.81 & 82.22 & 84.98 & 76.06 & 78.08 & 77.18 & 7.90 & 3.05 & 0.39 & 22.6 & 46.2 & 6.6 \\
SinLoRA & 82.8 & 69.45 & 59.99 & 85.01 & 81.45 & 84.61 & 75.22 & 78.67 & 76.34 & 7.36 & 2.21 & 0.30 & 734.6 & 61.6 & 7.1 \\
RandLoRA & 82.78 & 68.96 & 59.66 & 85.02 & 81.84 & 84.93 & 75.32 & 77.95 & 76.24 & 7.34 & 2.11 & 0.28 & 605.5 & 159.9 & 11.8 \\
Krona & 84.08 & 72.09 & 61.40 & 87.17 & 82.87 & 85.55 & 76.48 & 78.88 & 77.78 & 8.64 & 3.65 & 0.42 & 755.2 & 42.7 & 5.02 \\
KRAdapter & 83.64 & 73.03 & 61.95 & 87.85 & 82.79 & 85.72 & 76.5 & 79.32 & 78.17 & 8.20 & 4.04 & 0.49 & 920.9 & 9.8 & 2.8 \\
\midrule
FT & 85.05 & 68.13 & 60.30 & 86.00 & 83.41 & 86.14 & 77.28 & 75.74 & 76.71 & 9.61 & 2.58 & 0.27 & 758.8 & 1.0 & 1.0 \\
\midrule
\multicolumn{16}{l}{\textbf{ViT-H/14}} \\
\midrule
ZS & 77.94 & 59.36 & 66.53 & 89.29 & 77.59 & 81.30 & 70.93 & 84.74 & 75.74 & n/a & n/a & n/a & n/a & n/a & n/a \\
LoRA & 83.65 & 56.76 & 64.38 & 85.62 & 82.54 & 85.91 & 76.21 & 79.89 & 75.90 & 5.71 & 0.22 & 0.04 & 27.7 & 70.0 & 7.4 \\
SinLoRA & 83.55 & 58.72 & 65.93 & 87.90 & 82.85 & 86.02 & 76.4 & 81.69 & 77.07 & 5.61 & 1.40 & 0.29 & 968 & 90.6 & 8.1 \\
RandLoRA & 82.94 & 57.04 & 63.4 & 84.90 & 81.84 & 85.17 & 75.13 & 80.58 & 75.43 & 5.00 & -0.24 & -0.05 & 752.3 & 508.3 & 21.21 \\
Krona & 85.02 & 64.33 & 65.78 & 87.52 & 83.62 & 86.55 & 77.15 & 82.43 & 78.20 & 7.08 & 2.52 & 0.36 & 882.4 & 123.3 & 9.2 \\
KRAdapter & 84.57 & 65.67 & 67.15 & 89.01 & 83.38 & 86.51 & 76.96 & 83.23 & 78.84 & 6.63 & 3.17 & 0.48 & 1140.2 & 32.8 & 5.5 \\
\midrule
FT & 84.88 & 64.23 & 67.26 & 89.68 & 81.96 & 85.07 & 75.44 & 84.88 & 78.36 & 6.94 & 2.68 & 0.39 & 935.4 & 4.3 & 1.9 \\
\bottomrule
\end{tabular}
\caption{Detailed results on OOD generalization with efficient rank\label{tab:ooddetail}}
\end{table}

\section{Commonsense reasoning\label{app:commonsense}}
\subsection{Dataset details}
We test on $8$ commonsense reasoning datasets. These benchmarks encompass a range of cognitive skills, including answering yes/no questions  BoolQ~\citep{2019_arXiv_boolq}), addressing common-sense physics inquiries (PIQA~\citep{2020_AAAI_piqa}), understanding social dynamics (SIQA~\citep{2019_arXiv_siqa}), completing multi-choice scenarios (HellaSwag~\citep{2019_arXiv_hellaswag}), binary solutions to finish sentences (WinoGrande~\citep{2021_arXiv_winogrande}), tackling both simpler and more complex elementary science questions (ARC-e and ARC-c~\citep{2018_arXiv_ARC}), and engaging in multi-stage reasoning (OBQA~\citep{2018_arXiv_OBQA}). This collection of datasets presents different challenges, ranging from understanding the nuances of language and employing everyday knowledge to making inferences about the physical and social world. For a deeper exploration of these datasets, we redirect readers to the work of Hu et al.~\citep{2023_arXiv_llmadapters}.

\subsection{Training details}
The models are trained using the Transformers library from Hugging Face\footnote{https://huggingface.co}. We followed implementation specifics detailed by Albert et al.~\cite{2025_ICLR_RandLoRA}, whose code is publicly available\footnote{\url{https://github.com/PaulAlbert31/RandLoRA}}. The training lasts for four epochs, utilizing a learning rate of $1\times 10^{-4}$ and a base scaling coefficient of $2$ for $\alpha$ weights. To combat overfitting we use dropout with a probability of $0.05$ for each adapter layer. Unless otherwise specified, hyper-parameters were kept consistent across different architectures and algorithms.  We train on the multi-choice tasks SIQA, ARC-C, ARC-E and OBQA and test on all tasks.

\section{GLUE\label{app:GLUE}}
We further report results tuning RoBERTa~\cite{2019_arxiv_roberta} on the General Language Understanding Evaluation (GLUE)~\cite{2019_ICLR_glue} dataset (see appendix~\ref{app:GLUE}). We train for the SST-2, MRPC, COLA, QNLI, RTE and STS-N tasks. We report Matthew’s correlation for CoLA, Pearson correlation for STS-B, and accuracy for the remaining tasks.  We train the key and value matrix in the attention layers of a pretrained RoBERTa-large~\cite{2019_arxiv_roberta} network configuration with 355M parameters originally and perform 5 runs to report average performance and one standard deviation. We train each run for $10$ epochs with a learning rate of $10^{-4}$.Results are reported in Table~\ref{tab:glueresults} where we find that KRAdapter slightly outperforms other algorithms on average although results are very close. In this setting, the margin for improvement is small as the task is an easy binary classification. This translates to all PEFT algorithms producing results within an error margin of each other. KRAdapter however performs competitively in this setting as well.

\begin{table*}[htbp]
\centering
\small
\setlength{\tabcolsep}{2pt}
\begin{tabular}{lccccccccc}
\toprule
LoRA & 95.6 $\pm$ 0.2 & 88.7 $\pm$ 0.9 & 64.3 $\pm$ 1.2 & 94.6 $\pm$ 0.2 & 79.1 $\pm$ 4.0 & 91.8 $\pm$ 0.4 & 85.7 $\pm$ 0.9 \\
SinLoRA & 96.1 $\pm$ 0.1 & 88.9 $\pm$ 0.9 & 63.4 $\pm$ 0.9 & 93.6 $\pm$ 0.6 & 83.7 $\pm$ 0.4 & 91.8 $\pm$ 0.1 & 86.3 $\pm$ 0.2 \\
RandLoRA & 95.7 $\pm$ 0.3 & 88.7 $\pm$ 0.4 & 63.9 $\pm$ 1.3 & 93.9 $\pm$ 0.3 & 81.7 $\pm$ 2.3 & 91.8 $\pm$ 0.2 & 85.9 $\pm$ 0.3 \\
Krona & 95.8 $\pm$ 0.2 & 88.0 $\pm$ 0.8 & 59.6 $\pm$ 0.8 & 94.3 $\pm$ 0.2 & 78.7 $\pm$ 2.4 & 91.6 $\pm$ 0.3 & 84.7 $\pm$ 0.4 \\
KRAdapter & 95.9 $\pm$ 0.4 & 89.2 $\pm$ 0.6 & 64.6 $\pm$ 0.6 & 94.1 $\pm$ 0.3 & 82.5 $\pm$ 0.7 & 92.0 $\pm$ 0.3 & 86.4 $\pm$ 0.1 \\
\bottomrule
\end{tabular}
\caption{Results on GLUE datasets with the RoBERTa-large model.}
\label{tab:glueresults}
\end{table*}